\documentclass{article}

\usepackage[preprint]{Styles/neurips_2024}

\usepackage[utf8]{inputenc} 
\usepackage[T1]{fontenc}    
\usepackage{hyperref}       
\usepackage{url}            
\usepackage{booktabs}       
\usepackage{amsfonts}       
\usepackage{nicefrac}       
\usepackage{microtype}      
\usepackage{xcolor}         

\usepackage{amsmath}
\usepackage{amsthm}
\usepackage{mathtools}
\usepackage{here}
\usepackage{cancel}
\usepackage{natbib}
\usepackage{amssymb}

\usepackage{algorithm}
\usepackage{algorithmic}
\usepackage{bm}
\usepackage{times}
\usepackage{booktabs}
\usepackage{subcaption}
\usepackage{graphicx}
\usepackage{physics}
\usepackage{url}
\usepackage{multirow}
\usepackage{color}

\usepackage{varwidth}
\usepackage{array}
\usepackage{thmtools}
\usepackage{empheq}
\usepackage{wrapfig}
\usepackage{xcolor}    
\usepackage[normalem]{ulem}
\usepackage[all]{xy}    
\usepackage{scalerel}

\usepackage{tikz}
\usepackage{tikz-cd}
\usepackage{tikz-3dplot}
\usetikzlibrary{arrows.meta}
\usetikzlibrary{calc}
\usetikzlibrary{matrix,shapes,arrows,positioning}

\usepackage{xcolor}
\hypersetup{
    colorlinks=true,
    citecolor=blue,
    linkcolor=blue,
    urlcolor=teal,
    breaklinks=true
}
\usepackage{aliascnt}
\usepackage[framemethod=tikz]{mdframed}

\usepackage{tcolorbox}
\tcbuselibrary{breakable, skins, theorems, listings}

\usepackage[capitalize,noabbrev]{cleveref}

\newcounter{thm}
\theoremstyle{plain}
\newtheorem{lemma}[thm]{Lemma}
\newtheorem{corollary}[thm]{Corollary}
\theoremstyle{definition}
\newtheorem{definition}[thm]{Definition}
\newtheorem{example}[thm]{Example}
\newtheorem{remark}[thm]{Remark}

\newtcbtheorem[use counter= thm]{theorem}{Theorem}{colback=white, colframe=black, boxrule=0.8pt, fonttitle=\bfseries, label type=theorem, fontupper=\normalsize\itshape, breakable=true}{thm}

\newtcbtheorem[use counter= thm]{proposition}{Proposition}{colback=white, colframe=gray, boxrule=0.8pt, fonttitle=\bfseries, label type=proposition, fontupper=\normalsize\itshape,breakable=true}{prop}

\crefname{theorem}{Theorem}{Theorems}
\crefname{proposition}{Proposition}{Propositions}
\crefname{lemma}{Lemma}{Lemmas}
\crefname{corollary}{Corollary}{Corollaries}
\crefname{definition}{Definition}{Definitions}
\crefname{remark}{Remark}{Remarks}
\crefname{example}{Example}{Examples}
\crefname{subsubsection}{\S}{\S}
\crefname{subsection}{\S}{\S}
\crefname{section}{\S}{\S}
\crefname{appendix}{Appendix}{Appendices}
\crefname{chapter}{Chapter}{Chapters}
\crefname{table}{Table}{Tables}
\crefname{figure}{Figure}{Figures}
\crefname{algorithm}{Algorithm}{Algorithms}

\newcommand{\R}{{\mathbb{R}}}
\newcommand{\N}{{\mathbb{N}}}

\newcommand{\cF}{\mathcal{F}}
\newcommand{\Omegatilde}{{\widetilde\Omega}}

\newcommand{\Diri}{\mathrm{Dir}}
\newcommand{\Conv}{\mathrm{ConvHull}}



\newcommand{\la}{\left\langle}
\newcommand{\ra}{\right\rangle}
\newcommand{\Uniform}[1]{\mathrm{Unif}(#1)}
\newcommand{\Gaussian}[2]{\mathcal{N}(#1,#2)}
\newcommand{\Prob}[1]{\mathcal{P}(#1)}

\newcommand{\ggc}{GGc}

\newcommand{\given}[2]{
  \left(#1 \;\middle\vert\; #2\right)
}

\newcommand{\blank}{{\hspace{0.18em}\cdot\hspace{0.18em}}} 
\newcommand{\tcset}{I\times\Omega}
\newcommand{\cpsi}{\bar \psi}  
\newcommand{\psidist}{{Q}}
\newcommand{\ODEsolve}{\mathtt{ODEsolve}}
\newcommand{\ourmethod}{EFM}
\DeclareMathOperator*{\argmin}{arg\,min}
\DeclareMathOperator{\Expect}{\mathbb{E}}
\DeclareMathOperator{\Div}{\mathrm{div}}
\DeclareMathOperator{\supp}{\mathrm{supp}}

\makeatletter
\DeclareRobustCommand\mapstofill{%
  $\m@th
  {\mapstochar}%
  \smash-\mkern-7mu
  \cleaders\hbox{$\mkern-2mu\smash-\mkern-2mu$}\hfill
  \mkern-7mu
  \mathord\rightarrow
  $%
}
\makeatother

\DeclareSymbolFont{EulerExtension}{U}{euex}{m}{n}
\DeclareMathSymbol{\euintop}{\mathop} {EulerExtension}{"52}
\DeclareMathSymbol{\euointop}{\mathop} {EulerExtension}{"48}

\def\Set#1{\Setdef#1\Setdef}
\def\Setdef#1|#2\Setdef{\left\{#1\,\;\mathstrut\vrule\,\;#2\right\}}%

\def\Setdefs#1|#2\Setdefs{\left(#1\,\;\mathstrut\vrule\,\;#2\right)}%

\title{\mbox{Extended Flow Matching}: a Method of Conditional Generation with \mbox{Generalized Continuity Equation}}

\author{%
  Noboru Isobe\thanks{equal contribution}\\
  Graduate School of Mathematical Sciences\\
  University of Tokyo\\
  Tokyo, Japan \\
  \texttt{nobo0409@g.ecc.u-tokyo.ac.jp}\\
  \And
  Masanori Koyama$^\ast$\\
  Preferred Networks \\
  Tokyo, Japan \\
  \And
  Jinzhe Zhang\\
  Preferred Networks \\
  Tokyo, Japan \\
  \AND
  Kohei Hayashi \\
  Preferred Networks \\
  Tokyo, Japan \\
  \And
  Kenji Fukumizu \\
  The Institute of Statistical Mathematics / Preferred Networks\\
  Tokyo, Japan \\
}

\begin{document}

\maketitle

\begin{abstract}
    The task of conditional generation is one of the most important applications of generative models, and numerous methods have been developed to date based on the celebrated flow-based models. 
    However, many flow-based models in use today are not built to allow one to introduce an explicit inductive bias to how the conditional distribution to be generated changes with respect to conditions. This can result in unexpected behavior in the task of style transfer, for example. 
    In this research, we introduce extended flow matching (EFM), a direct extension of flow matching that learns a \textit{matrix field} corresponding to the continuous map from the space of conditions to the space of distributions. 
    We show that we can introduce inductive bias to the conditional generation through the matrix field and demonstrate this fact with MMOT-EFM, a version of EFM that aims to minimize the Dirichlet energy or the sensitivity of the distribution with respect to conditions. 
    We will present our theory along with experimental results that support the competitiveness of EFM in conditional generation.
\end{abstract}

\section{Introduction}\label{sec:intro}
Conditional generation is a task of generative models with significant importance in industrial and scientific applications, and it can be more mathematically described as a generation of random function $\psi\colon\Omega\ni c \mapsto x_c \in D$ that maps a ``condition'' to a data. 
Historically, such as in VAE~\citep{kingma2013auto} or GAN~\citep{goodfellow2020generative}, this random $\psi$ has been written in the form of generator $f(c, z) = \psi(c)$ with $z$ being a random sample generated from uninformative distribution. 
Conditional generation with continuous $c$ is of specific importance in applications like molecule design that involves inverse problems, and such studies have motivated studies such as \cite{ding2021ccgan}.

The focus of this paper is the conditional generation with the Flow matching method for continuous conditions. 
Branching from the work of \citet{lipman2023flow}, Flow Matching (FM) methods have emerged recently as a simulation-free alternative to the family of diffusion models \citep{ho2020denoising, sohl2015deep, song2020denoising}. 
In particular, with the development of powerful techniques such as OT-CFM \citep{tong2023improving}, FM is beginning to extend over a range of applied fields \citep{davtyan2023efficient, gebhard2023inferring, bose2023se,klein2023equivariant}, becoming utilized for various purposes of conditional generation. 
By aiming to minimize the kinetic \textit{work} consumed in transporting the source to the target distribution, these methods ensure inference of higher quality with low computational cost and stable training. 

However, OT-CFM literally does not introduce any inductive bias to how $\psi$ behaves with respect to the perturbation in $c$, while it uses the inductive bias regarding the energy consumed in interpolating between source and target distributions. 
When naively applied to conditional generation for continuous conditions, interpolations and style transfer may result in \textit{unintuitive} outcome.
For example, consider the example of the conditional generation problem with four conditional distributions with two clusters each (\cref{fig:clusters}). 
\begin{figure}[htbp]
\begin{center}
   \includegraphics[scale=0.35]{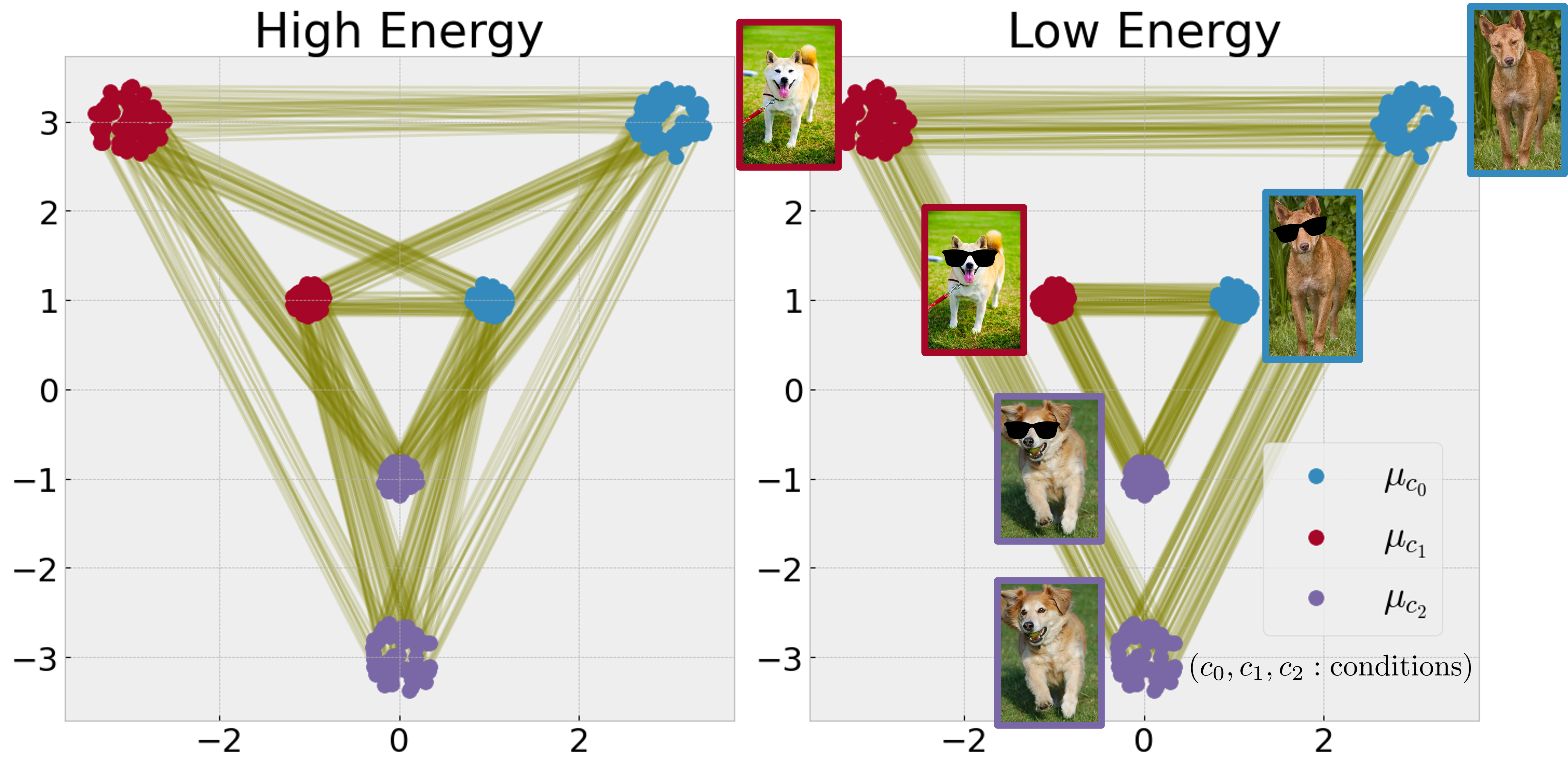}
 \caption{Across-condition transfer by the conditional generative model trained with three conditional distributions with two clusters each ($\mu_{c_0}$, $\mu_{c_1}$, $\mu_{c_2}$). When the average sensitivity of the distribution with respect to $c$ (Dirichlet energy) is not optimized, the inner cluster may mix with the outer cluster (left). Meanwhile, if the energy is optimized (right), the transfer would respect the separation of inner vs outer clusters. }
\end{center}
 \label{fig:clusters}
\end{figure}
The cluster in this example can be another hidden axis of condition (inner, outer). 
When we apply the style transfer of the OT-CFM model by continuously interpolating the condition with a fixed source element, the transfer trajectory ``mixes'' the inner cluster with the outer cluster, which may betray our inductive bias in predicting the interpolation. 
Indeed, this happens because OT-CFM literally does not introduce any inductive bias to how $\psi$ behaves with respect to the perturbation in $c$, while it uses the inductive bias regarding the energy consumed in interpolating between source and target distributions.

In this paper, we introduce extended flow matching (EFM), a direct extension of FM that allows the user to introduce an inductive bias regarding the aforementioned $\psi$.  
While FM designs the interpolation between source and target with a vector field, EFM also designs the interpolations conditional targets as well through the learning of \textit{matrix} field.
In particular, we present Multi-marginal optimal transport (MMOT) EFM, which extends OT-CFM with the aim of minimizing $\Expect[\|\nabla_c \psi \|^2]$ or the sensitivity of $\psi$ with respect to the change in condition.   

We summarize our contributions below:

\begin{enumerate}
\item We present \ourmethod, an algorithm to learn the matrix field that, through a generalized continuity equation \citep{LAVENANT2019688}, corresponds to a continuous map interpolating among source and conditional target distributions.
\item Through \ourmethod, one can introduce inductive bias to how the generation changes with respect to continuous conditions. In particular, we present MMOT-\ourmethod, a version of \ourmethod\ that aims to minimize the energy required to transport one condition to another.
\item We prove that \ourmethod\ can be described through per-example/conditional formulation and establish it as a direct generalization of the original flow-matching algorithm. 
\item We demonstrate that \ourmethod\ performs competitively on the task of conditional generation and also demonstrate that our matrix field simultaneously derives the vector field along the direction of the condition space, which can be used for the task of style transfer. 
 \end{enumerate}

We begin by first describing FM and OT-CFM 
in \cref{sec:usual_Flow_matching} , and describe in \cref{sec:EFMtheory} our theory underlying the algorithm of extended flow matching.  
We then introduce our \ourmethod\ algorithm in \cref{sec:EFMalg}, discuss the related works in \cref{sec:related}, and demonstrate its application in \cref{sec:exp}.




\subsection*{Notation}\label{subsec:Notation}
Let us use $\blank$ to denote a placeholder, $\norm{\blank}$ to denote the Euclidean norm, and $0_k\coloneqq(0,\dots,0)^\top\in\R^k$ to denote the zero vector.
We denote by $\Prob{M}$ the space of probability distributions on a metric space $M$, and denote by 
$\delta_x\in\Prob{M}$ the delta distribution supported on $x\in M$.
    For a distribution $\mu\in\Prob{M}$ on $M$ and a vector-valued function $f$ on $M$, we denoted by $\Expect_{X\sim\mu} [f(X)]$ the expectation of a random variable $f(X)$, where $X\sim\mu$ is a random variable following $\mu$.

We also denote $I\coloneqq\qty[0,1]$ and $\qty[m:n]\coloneqq\{m,m+1,\dots,n\}$ for $m$, $n\in\N$ such that $m<n$.
For a function $g$ on $I$, we write $\dot{g}(t)$ for the derivative $\dv{g}{t}{(t)}$ with respect to time $t\in I$.
Further, we let $D\subset\R^d$ be the data space. 
For any subscript $\xi$, we will denote by $p_\xi$ the density of a probability distribution $\mu_\xi$ on $D\subset\R^d$, i.e.,
$
    \mu_\xi(\dd x)=p_\xi(x)\dd x
$
in a measure-theoretic notation.
In the following mathematical discussion, we will assume that any probability distribution has a density, but this assumption is superficial and is used only for simplicity of explanation. 

\section{Preliminaries}\label{sec:usual_Flow_matching}
 \subsection{Flow Matching (FM)} \label{sec:FM}
\noindent \textbf{Continuity Equation}:     As a method of generative modeling, the goal of FM is to learn a map that transforms a source distribution to a target distribution in the form of $\mu\colon [0,1] \to \Prob{D}$, where $D$ is the space of dataset.   
Instead of learning $\mu$ directly, flow matching as a method learns a vector field $v\colon [0, 1] \times D \to \R^d$ such that the \emph{continuity equation} (CE)
\begin{equation}
\partial_tp_t(x)+\Div_x(p_t(x)v(t,x))=0.\label{eq:CE}
\end{equation} 
holds with respect to the density $p_t$ of $\mu_t$, and we use this $v$ for the sample generation.  \\
\noindent \textbf{Inference}: $X_1 \sim \mu_1$ can be sampled by solving the ODE with  $\dot{X}(t)=v(t,X(t))$, $X(0)\sim p_0$.  

\subsection{OT-CFM} \label{sec:OT-CFM}
\textbf{Objective energy}: OT-CFM, in particular, can be said to minimize the Dirichlet energy, or the energy of transport for $\mu$ conditional to the boundary condition $\mu_0 = \mu_{\mathrm{source}}, \mu_1 = \mu_{\mathrm{target}}$. 
Formerly, Dirichlet or the kinetic energy of the curve $\mu$ can be written as 
\begin{align}
    \Diri(\mu) \coloneqq \inf_{v\colon I \times D \to \R^d} \Set{\frac12\iint_{I\times D}\|v(t, x)\|^2p_t(x)\dd x\dd t  | \text{The pair }(p,v)\text{ satisfies }\eqref{eq:CE}  }. \label{eq:dir_1d} 
\end{align}
\textbf{Objective function}: To derive the algorithm used in OT-CFM, let us present $\mu$ as 
\begin{align}
\mu^\psidist = \Expect_{\psi\sim Q}[\mu^\psi],\label{eq:superposition} 
\end{align}
where $\mu^\psi_t(\dd x)  = \delta_{\psi(t)}(\dd x)$ is the point-mass distribution at $\psi(t)$ and $\psidist$ is a distribution over a space $H(I; D)\coloneqq\Set{\psi\colon I\to D|\psi\text{ is differentiable}}$ of paths that maps time $t$ to an instance $x \in D$.
The random $\psi$ that appears here is indeed 
analogous to $\psi$ in \cref{sec:intro}, except that the space $\Omega$ is replaced with time interval $I$.
Amazingly, because $\Diri$ turns out to be convex, we can bound 
$\Diri(\mu^Q)$ from above by 
$\int \mathrm{Dir}(\mu^\psi) Q(\dd \psi) = \iint \| \dot \psi(t) \|^2 \dd t Q(\dd \psi)$, and its minimization with respect to $Q$ conditional to $\mu^Q_0 = \mu_{\mathrm{source}}$, $\mu^Q_1 = \mu_{\mathrm{target}}$ turns out to be concentrated on the set of ``straight lines'' $\psi\given{t}{x_1, x_2}= t x_2  + (1-t)x_1  $ between joint samples $(x_1, x_2)$ from the target and the source, allowing the \textit{parametrization} of $Q$ with the joint distribution $\pi$ with marginals $\mu_{\mathrm{source}}$ and $\mu_{\mathrm{target}}$ \citep{Brenier2003, AGS}
This would allow us to write $\| \dot\psi\given{t}{x_1,x_2} \|^2  =   \| x_1 - x_2 \|^2$ for the optimal $Q^*$. 
This would reduce the optimization with respect to $Q$ to the classic optimal transport problem for the joint probability $\pi$ with cost $c(x,y) = \|x -y \|^2$, which can be approximated through batches.
Following the same logic as in \cite{kerrigan2023functional}, or our later theorem (\cref{thm:PMA}) that generalizes the denoising score matching in the flow setting, the flow corresponding to  $\mu_{{Q}^\ast}$ can be obtained as the minimizer of
\begin{align}
\Expect_{\psi\sim Q^\ast} [\|v(t, \psi(t)) - \dot\psi(t)\|^2 ]=  \Expect_{(x_1, x_2) \sim \pi^*}   [\|v(t, \psi(t)) - \dot\psi\given{t}{x_1, x_2}\|^2].
\end{align}
This derives the learning of $v(t, \psi(t))$ through a neural network $v_\theta$ as shown in \cref{alg:FM_train}.
Indeed,  Dirichlet energy that OT-CFM is aiming to minimize is a form of inductive bias regarding the continuity of the \textit{generation} process with respect to time $t$.  
In naive application of OT-CFM to conditional generation, $\psi(t)$ will be replaced with $\psi(t, c)$, but again, the energy in this situation only considers $\| \partial_t \psi(t, c)\|^2$.

\section{Theory of \ourmethod}\label{sec:EFMtheory}

In this section, we extend the standard FM theory to consider conditional probability with conditions $c$ within a bounded domain $\Omega\subset\R^k$.
In this section, we extend the standard FM theory to consider conditional probability with conditions $c$ within a bounded domain $\Omega\subset\R^k$.
Let $p_c(x)\coloneqq p\given{x}{c}$ be the unknown target conditional probability density, and let $p_{0,c}(x)\coloneqq p_0\given{x}{c}$ be a user-chosen tractable conditional density given $c=(c^i)_{i\in\qty[1:k]}=(c^1,\dots,c^k)\in\Omega$, such as normal distributions with mean and variance parameterized by $c$.  
We will use the notation in \hyperref[subsec:Notation]{the previous section}.
That is, denote by $\mu_c$ and $\mu_{0,c}$ the distribution of the probability density function $p_c$ and $p_{0,c}$, respectively. 

\begin{figure}[htb]
\centering
\includegraphics[scale=.8]{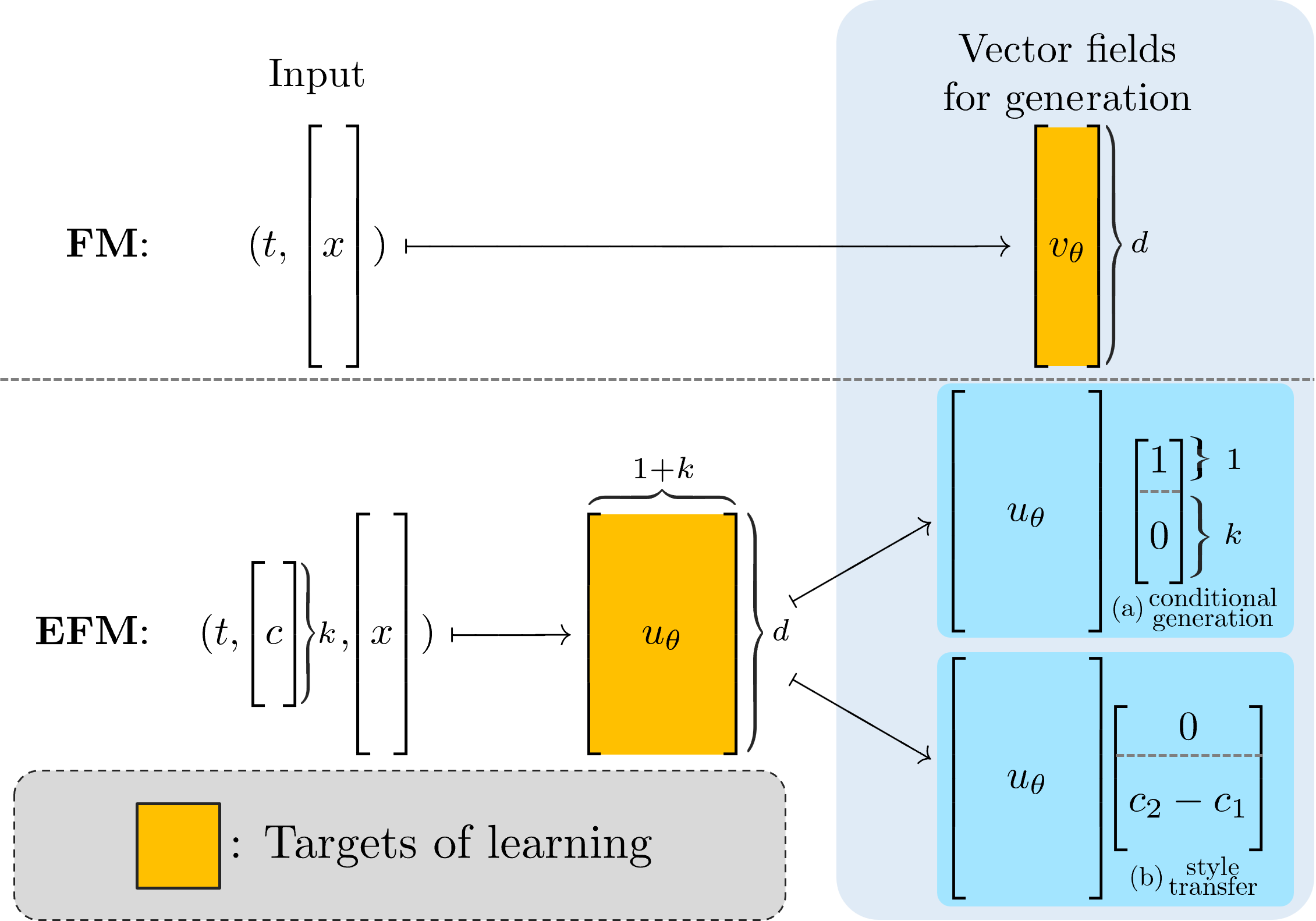}
\caption{
Inferences of FM and \ourmethod.
} \label{fig:diff} 
\end{figure}
\subsection{Extension of FM}\label{subsec:gen}
 We will present this subsection in parallel with \cref{sec:FM}. \\
\noindent \textbf{Generalized Continuity Equation}: We 
directly extend the interpretation of FM by extending the domain of $\psi$ in \eqref{eq:superposition} from $I$ to $\tcset$, where $\Omega$ is the space of conditions.
For brevity, instead of using explicit $\tcset$, we would like to use a general bounded domain $\Xi$ in Euclidean space as an analogue of $\Omega$ of the previous section and analogously set the goal of EFM to the learning of $\mu\colon \Xi \to \Prob{D}$.
Now, just like FM, instead of learning $\mu$ directly, EFM aims to learn a \textit{matrix} field $u\colon\Xi \to \R^{d \times \dim\Xi}$ such that 
\textit{generalized CE}~\citep{Brenier2003,LAVENANT2019688}
\begin{align}
    \nabla_{\xi} p_{\xi}(x) +\Div_x(p_{\xi}(x)u(\xi,x))=0~
    ((\xi,x)\in\Xi\times D&),
    \label{eq:GCE}
\end{align}
holds for the density $p_{\xi}$ of $\mu_{\xi}$. 

\noindent \textbf{Inference}: Inference based on $u$ is slightly more complicated than in FM, which provides a single vector field on which to integrate the ODE. 
When $\Xi = \tcset$,  the generation for condition $c$ will be done by transforming $\mu_{0, c} \to \mu_{1,c}$, and the transfer from $c$ to $c^\prime$ by transforming $\mu_{1,c} \to \mu_{1,c^\prime}$. 
These are both done through integrating the matrix field along the path in $\tcset$.  
More precisely, the following result justifies our use of the matrix field $u$ in \eqref{eq:GCE} to achieve the goal of conditional generative modeling:

\begin{proposition}{GCE generates $\boldsymbol{\gamma}$-induced CE}{gamma_wise_gen}
    Let $\mu\colon\Xi \to\Prob{D}$ and $u\colon \Xi\times D \to\R^{d\times\dim\Xi}$ be a probability path and a matrix field, respectively, that satisfy \eqref{eq:GCE}.
    Then, for any differentiable path $\gamma\colon I\to\Xi$, the $\gamma$-induced probability path $\mu^\gamma\coloneqq \mu\circ\gamma$ and the $\gamma$-induced vector field $v^\gamma\colon I \times D \ni(s,x)\mapsto u(\gamma(s), x)\dot\gamma(s)\in\R^d$ satisfy the continuity equation, i.e., the density $p^\gamma$ of $\mu^\gamma$ and $v^\gamma$ satisfy
    $        \partial_sp^\gamma_s(x)+\Div_x(p^\gamma_s(x)v^\gamma(s,x))=0.
    $ 
\end{proposition}

\begin{figure}[htbp]
        \centering
        \includegraphics[keepaspectratio, scale=1]{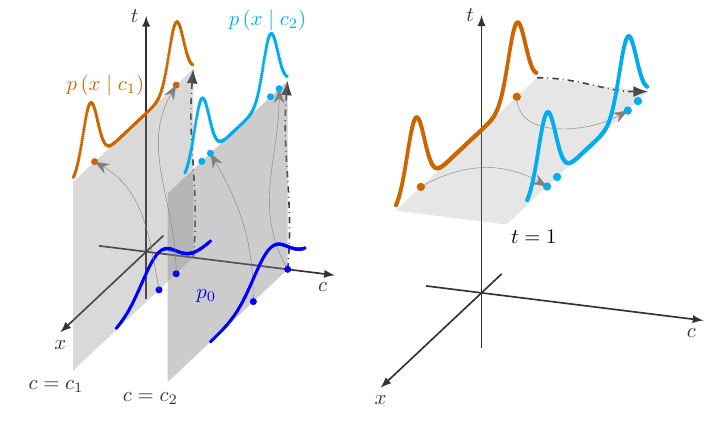}
        \\
        \includegraphics[keepaspectratio, scale=1]{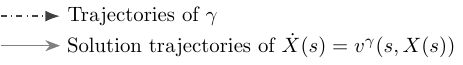}
    \caption{Visualization of the flow for (a) conditional generation along $\gamma^{c_1}$ and $\gamma^{c_2}$ (\cref{alg:gen}), and
(b) style transfer along $\gamma^{c_1\to c_2}$ (\cref{alg:transfer}).} 
\label{fig:gen&trans}
\end{figure}

The rigorous version of \cref{prop:gamma_wise_gen} is given in \cref{prop:induced_flow} in the Appendix.
\cref{prop:gamma_wise_gen} shows that the flow on $D$ corresponding to an arbitrary probability path on $\Set{\mu_{\xi}\in\Prob{D} | \xi \in \Xi }$ can be constructed from the $\gamma$-induced vector field obtained from multiplying the matrix $u$ to the vector $\dot{\gamma}$.
Thus, once the matrix field $u$ is obtained, the desired vector field $v^\gamma$ is to be calibrated by choosing an appropriate $\gamma$ that suits the purpose of choice.
When the pair of $p_{\xi}$ and $u_{\xi}$ satisfies GCE \eqref{eq:GCE}, the designs of $\gamma$ in the following two examples possess significant practical importance (See \cref{fig:diff} and \cref{fig:gen&trans} ): 


\begin{example}[Conditional generation]\label{ex:cond_gen}
    When the goal is to sample from the unknown conditional distribution $\mu_{c_\ast}$ given condition $c_\ast\in\Omega$, 
    we can choose $\gamma^{c_\ast}\colon I\to\tcset$ such that $\gamma^{c_\ast}(1)=(1,c_\ast)$; 
    typically, we can set $\gamma^{c_\ast}(s)=(s,c_\ast)$ for $s\in I$.
    Then, by virtue of \cref{prop:gamma_wise_gen} and the continuity equation \eqref{eq:CE}, we only need to compute the flow $\phi$ by solving the ODE
    $
    \left\{
    \begin{aligned}
        &\dot{\phi}_s(x_0)
        = u(s, c_\ast,{\phi}_s(x_0)) 
        \sbmqty{1\\0_k}
        \ (s\in I) , \\
        &x_0\sim \mu_{0,c_\ast},
    \end{aligned}
    \right.
    $
    and obtain samples $\phi_1(x_0)$ from $\mu_{1,c}=\mu_c$.
    The trajectories in the front and rear plane of (a) in \cref{fig:gen&trans} respectively represent the flows corresponding to this example with $c_\ast = c_1$ and $c_\ast = c_2$.
\end{example}
\begin{example}[Style transfer]\label{ex:style_trans}
    When the goal is to transform a sample generated from $\mu_{c_1}$ to a sample of another distribution $\mu_{c_2}$ given $c_2\in\Omega$, 
    we may choose $\gamma^{c_1\to c_2}\colon I\to\tcset$ satisfying $\gamma^{c_1\to c_2}(0)=(1,c_1)$ and $\gamma^{c_1\to c_2}(1)=(1,c_2)$.
    For example, we can set $\gamma^{c_1\to c_2}(s)=(1,(1-s)c_1+sc_2)$ for $s\in I$.
    In this case, we only need to solve the ODE 
    $
    \left\{
    \begin{aligned}
            &\dot{\phi}_s(x_0)
            = u(1, \gamma^{c_1\to c_2}(s),\phi_s(x_0) )
            \sbmqty{0\\c_2-c_1}
            \ (s\in I),\\  
            &x_0\sim\mu_{c_1}.
    \end{aligned}
    \right.
    $
    The solution trajectories in (b) in \cref{fig:gen&trans} represent the flows corresponding to this style transfer.
\end{example}

\subsection{Convex energy and MMOT-EFM} 
Now we extend the arguments in \cref{sec:OT-CFM} to EFM. \\
\textbf{Objective energy}: Just like in \cref{sec:OT-CFM},  we use the representation of $\mu$ as \eqref{eq:superposition} through a distribution $Q$ over a space $H(\Xi; D)$ of differentiable maps $\psi$ from $\Xi$ to $D$.
Now, the construction of EFM allows us to introduce inductive bias regarding a property of $\psi\colon \Xi \to D$ and hence how $\mu$ behaves with respect to $\xi$.  In particular, 
if a given energy $\mathcal{E}$ with respect to $\mu^\psi$ is convex, then by simple Jensen's inequality we can bound  $\mathcal{E}(\mu)$ from above by $\Expect_{\psi\sim Q} [\mathcal{E}(\mu^\psi)]$. 
Please also see \cref{prop:Dirichlet_Jensen,prop:Dirichlet_bound} for more precise statements
of these results. In MMOT-EFM, we consider the case in which $\mathcal{E}$ is the following generalization of \eqref{eq:dir_1d}.

A generalized version of Dirichlet energy \citep{LAVENANT2019688} of a function $\mu\colon\Xi\to\Prob{D}$ is given by
\begin{align}
    \Diri(\mu) \coloneqq \inf_{v\colon \Xi \times D \to \R^d} \Set{\frac12\iint_{\Xi \times D}\|u(\xi, x)\|^2 p_\xi(x)\dd x\dd \xi  | \text{The pair }(p,u)\text{ satisfies }\eqref{eq:GCE}  } \label{eq:dir_multi} 
\end{align}
where $p_\xi$ is the density of $\mu_\xi$. 
This energy is of great practical importance because it also measures how great $\mu$ changes with respect to $\xi$ \citep{Reshetnyak1997}. \\
\textbf{Objective function}:  Unfortunately, unlike in the case of OT, the energy-minimizing $\mu$ that can be written as $\mu = \mu^Q\coloneqq\Expect_{\psi\sim Q}[\mu^\psi]$ is not necessarily achieved with $Q$ concentrated on straight ``planes'' interpolating joint samples from $\{\mu_{\xi}\}$, so we choose to constrain the search of $\psidist$ to a specific subspace $\mathcal{F}$ of $H(\Xi; D)$, such as RKHS.
In this search, we indeed also require $Q$ to satisfy the boundary condition (BC) that
\begin{align}
\Expect_{\psi\sim Q}\qty[\delta_{\psi(\xi)}]=\mu_{\xi}~(\xi\in A), \label{eq:BC_psi}
\end{align}
where $A\subset\Xi$ is a finite set for which $\mu_{\xi}$ $(\xi\in A)$ is either known or observed. 
To this end, if $\boldsymbol{x}_A\coloneqq (x_{\xi})_{\xi \in A}$ for $A \subset \Xi$ is a joint sample with  $x_{\xi} \sim \mu_\xi$, then let $\phi\colon D^{|A|} \to \mathcal{F}$ defined by the regression 
\begin{align}
\psi(\cdot) \coloneqq  \phi\given{\cdot}{\boldsymbol{x}_A} \in \argmin_{f \in \mathcal{F}} \sum_{\xi \in A}\| f(\xi) - x_{\xi}\|^2. \label{eq:regression}
\end{align}
If $\pi$ is a joint distribution on $D^{|A|}$, the parametrization $Q = \phi_\# \pi$ allows us to bound the energy from above in the following way: 
\[
    \inf_Q \Diri(\mu^Q) \leq \inf_{Q} \int \Diri(\mu^\psi) Q(\dd\psi) = \inf_{Q} \int \|\nabla_\xi \psi\|^2 Q(\dd\psi) \leq  \inf_{\pi} \int \|\nabla_\xi \phi( \cdot | \boldsymbol{x}_A) \|^2 \pi(\dd\boldsymbol{x}_A ) 
\]
Now observe that the upper bound is the form of a marginal optimal transport problem about $\pi$ with marginals $\mu_A$  and $c ( \boldsymbol{x}_A) = \int_\Xi\|\nabla_\xi \phi\given{\xi}{\boldsymbol{x}_A}) \|^2\dd\xi$, whose solution $\pi^*$ can be approximated with batch as in the OT-CFM case. 
See \cref{tab:compare_EFM_FM} for the parallel correspondence between MMOT-EFM and OT-CFM.

\begin{table}[htbp]
\centering
\caption{Constructions of $\psi\colon[0,1] \to D$ and $\cpsi\colon\Omega \to D$ and $\pi$ in OT-CFM and MMOT-EFM. Note that they agree when $\mathcal{F}$ is a set of linear functions from $\Omega$ to $D$  and when $\Omega = [0, 1] \subset \R $.}
\label{tab:EFM_comparison}
\setlength\tabcolsep{1pt} 
\begin{tabular}{@{}lcc@{}}
\toprule
 & OT-CFM  &  MMOT-EFM \\ \midrule
Interpolator & $\psi\given{t}{x, y} =  t x + (1-t)y$ & $\cpsi\given{\cdot}{\boldsymbol{x}=(x_i)_i} \in \argmin\limits_{\phi \in \mathcal{F}} \sum_{i} \norm{\phi(c_i) - x_i}^2$ \\
Cost 
& 
\begin{tabular}{rc}  
&$\iiint\limits_{[0, 1]\times D \times D}\hspace{-10pt}\|\dot \psi\given{t}{ x, y}\|^2 \dd{t}\pi(\dd x, \dd y) $\\
$(=$&$\iint\limits_{D \times D}  \norm{x-y}^2 \pi(\dd x, \dd y))$ \end{tabular}  & $\displaystyle\iint\limits_{\Omega\times D^{\abs{C}}} \norm{\nabla_c \cpsi\given{c}{\boldsymbol{x}} }^2 \dd{c} \pi(\dd \boldsymbol{x})  $     \\ 
\bottomrule
\end{tabular} \label{tab:compare_EFM_FM}
\end{table}

The \cref{thm:PMA} which we provide in the following allows us to train the $u$ corresponding to $\pi^\ast$ and hence $\mu^{Q^\ast}$ as the minimizer of 
\begin{align}
\Expect_{Q^\ast} [\|u(t, \psi(t)) - \nabla_\xi \psi(\xi)\|^2 ]=  \Expect_{\boldsymbol{x}_A \sim \pi^\ast}   [\|u(t, \psi(t)) - \nabla_\xi  \phi( \cdot | \boldsymbol{x}_A) \|^2] \label{eq:EFM-obj}
\end{align}
which we would use as the objective function of MMOT-EFM. 

\begin{theorem}{Fundamental theorem for \ourmethod}{PMA}
Assume we have a random path $\psi\sim \psidist\in\Prob{H(\Xi;D)}$ that satisfies \eqref{eq:BC_psi} and let $\mu_{t,c}=\Expect_{\psi\sim\psidist}\qty[\delta_{\psi(\xi)}]$ for $\xi\in\Xi$.
For neural networks $u_\theta$, set
\begin{equation}
    \mathcal{L}^{\prime}(\theta)=\int_{\Xi}\Expect_{\psi\sim\psidist}\norm{u_\theta(\xi,\psi(\xi))-\nabla_{\xi}\psi(\xi)}^2\dd \xi.  \label{eq:EFM_cond_obj} 
\end{equation}
If there exists a matrix field $u\colon\Xi\times D\to\R^{d\times(1+k)}$ satisfying \eqref{eq:GCE}, 
then it holds that
$\nabla_\theta\mathcal{L}(\theta)=\nabla_\theta\mathcal{L}^\prime(\theta)$ for $\theta\in\R^p$.
Here, we set
\[
    \mathcal{L}(\theta)\coloneqq\int_{\Xi}\Expect_{x\sim\mu_\xi}\norm{(u_\theta-u)(\xi,x)}^2\dd \xi.  
\]
\end{theorem}

This result follows from \cref{lem:PMA} in the Appendix.

\section{Training algorithm}\label{sec:EFMalg}

In this section, we leverage the \ourmethod\ theory of \cref{sec:EFMtheory} to construct an algorithm for learning $u_\theta$ in \cref{prop:gamma_wise_gen}, which can be used for conditional generation tasks as well as for style transfer.
We summarize the training algorithm in \cref{alg:train}. 

Because EFM is a direct extension of FM, our algorithm roughly follows the same line of procedures as that of FM (\cref{alg:FM_train}): (a) sampling data, (b) constructing the supervisory signal $\nabla \psi$, and (c) updating the network by averaged loss.
However, in our algorithm, the domain of $\psi$ is $I \times \Omega$ as opposed to just $I$.  
We developed our algorithm so that, when it is applied to the unconditional case, the trained model agrees with FM.
Although the general EFM, as opposed to MMOT-EFM, does not necessasrily need parametrize $Q$ with respect to joint distribution $\pi$,
in this paper we focus ton he procedure that uses the joint distribution $\pi$ and $\psi$ in the form of \eqref{eq:regression} and \eqref{eq:EFM-obj}.

      \begin{algorithm}[H]
        \caption{Algorithm of EFM}
\begin{algorithmic}[1]
\renewcommand{\algorithmicrequire}{\textbf{Input:}}
 \REQUIRE Conditions $C \subset\Omega$, set of datasets $D_c \subset D$ ($c \in C$), network $u_\theta\colon\tcset\times D \to \R^{d \times (1+k)}$, source distributions $p_0\given{\blank}{c}$ ($c\in C$)
 \renewcommand{\algorithmicensure}{\textbf{Return:}}
 \ENSURE $\theta\in\R^p$ 
  \FOR {each iteration}
   \item[] \texttt{\color{teal}\# Step 1: Sample}
  \STATE Sample $C_0$ from $C$, $B_{0,c}$ from $p_0\given{\cdot}{c}$ and $B_{1,c}$ from $D_c$ ($c \in C_0$). Put $B^0 \coloneqq \{B_{0,c}\}_{c\in C_0}$, $B^1 \coloneqq \{B_c\}_{c\in C_0}$
 \item[] \texttt{\color{teal}\# Step 2: Construct $\psi\colon I\times\Omega\to D$} 
   \STATE Construct a transport plan $\pi$ among $B^0$ and $B^1$ \texttt{\color{teal}\#\cref{sec:EFMalg}}\\
   
     \STATE Sample $(x_{t,c})_{t,c}\sim\pi$
   \STATE Define $\psi\colon\tcset\to D$ s.t. \eqref{eq:psi_tc}
  \STATE Sample $t\sim\Uniform{I}$, $c\sim\Uniform{\Conv(C_0)}$
  \STATE Compute 
\[
    \begin{aligned}
    \psi_{t,c}&\coloneqq\psi(t,c)\\
    \nabla\psi_{t,c}&\coloneqq\nabla_{t,c} \psi(t,c)
    \end{aligned}
\]
  \STATE Update $\theta$ by $ \nabla_\theta \|u_\theta(t, {c}, \psi_{t,c}) -  \nabla\psi_{t,c} \|^2$
  \ENDFOR
\end{algorithmic}
      \end{algorithm}
 
\textbf{Step 1 Sampling from Datasets}:  
Our objective begins from the sampling of $\psi$, whose jacobian serves as the supervisory signal in the objective \eqref{eq:EFM-obj}.  
In order to sample $\psi$, we construct $\psidist$ from a joint distribution $\pi$ defined over $D^{2 N_c}$ with marginals that are approximately $\{ \mu_{t, c} \}_{t \in \{0,1\},c \in C_0}$.
To this end, we begin by randomly choosing a subset $C_0\coloneqq\{c_i\}_{i=1}^{N_c}$ from $C$ so that $C_0$ consists of close points. We then sample a batch $B_{0,c}$ from $\mu_{0,c}$ and $B_{1,c}$ from $D_c$ for each $c \in C_0$.
For the reason we describe at the end of this section,  we chose $\mu_{0,c} = \operatorname{Law}(R(c) + z)$ with $z$ being a common Gaussian component, and  $R\colon \Omega \to D$ is regressed from $\{ (c_i, \operatorname{Mean}[D_{c_i}]) \}_i$ by a linear map. 
We choose this option because it theoretically aids us in reducing $\Diri(\mu)$ (See \cref{prop:Dirichlet_bound}).

\textbf{Step 2 Constructing the supervisory paths}: 
Given the samples $B = \{ B_{t, c} \}_{ t \in \{0,1\},c \in C_0}$, we sample  $\{ x_{t, c} \}_{c \in C_0, t \in \{0,1\}}$ from a joint distribution $\pi$ over $D^{2N_c}$ with support on $B$. 
In MMOT-EFM, as an internal step, we train the joint distribution $\pi$ with $c ( \boldsymbol{x}_A) =\int_\Xi \|\nabla_\xi \phi\given{\xi}{\boldsymbol{x}_A}\|^2\dd{\xi}$ with $\phi$ solved analytically for \eqref{eq:regression} (e.g. Kernel Regression, Linear regression, etc).
When possible, the regression function may be chosen to reflect the prior knowledge of the metrics on $\Omega$ by extending the philosophy of \citet{chen2023riemannian} to the space of conditions. 
In practice, however, the computational cost of MMOT scales exponentially with the number of marginals, so we optimize the joint distributions over $B_1 = \{ B_{1, c} \}_{1,c \in C_0}$ only and couple the analogous $B_0$ to $B_1$ via the usual optimal transport. 
Please see Appendix \ref{sec:appendix-MMOT-empirical} for a more detailed sampling procedure. 
Now, given a joint sample $\{ x_{t, c} \}_{c \in C_0, t \in \{0,1\}}$, we construct $\psi$ as 
\begin{align}
\psi\given{t, c}{x_{0,c},\boldsymbol{x}_{C_0}} = (1-t)x_{0,c} + t \cpsi \given{c}{\boldsymbol{x}_{C_0}}  \label{eq:psi_tc}
\end{align}
where $\cpsi \given{c}{\boldsymbol{x}_{C_0}}$ is the solution of the kernel regression problem for the map $T\colon\R^k\ni c \mapsto x_{1,c}\in\R^d$ with any choice of kernel on $\R^k$. 
Note that this construction of $\psi$ satisfies the boundary condition \eqref{eq:BC_psi} with $A = \{0,1\} \times  C_0$, and generalizes the $\psi$ used in OT-CFM. 


\textbf{Step 3 Learning the matrix fields}: 
Thanks to the result of \cref{thm:PMA}, we may train $u_\theta\colon\tcset \to \R^{d \times (1+k)}$ via the loss function being the Monte Carlo approximation of \eqref{eq:EFM_cond_obj}.


\section{Inference}
The sampling procedures for style transfer and conditional generation respectively follow \cref{ex:style_trans} and \cref{ex:cond_gen}. 
For the task of style transfer from $c_0$ to $c_\ast$, we use the flow along the path $\mu_{1,c_0} \to \mu_{1,c_\ast}$.
For the task of conditional generation with target condition $c_\ast$,  we use the flow along $\mu_{0,c_\ast} \to p_{\mu,c_\ast}$.  
See \cref{alg:gen,alg:transfer} for the pseudo-codes.
When generating a sample for $c^* \not\in C$, the source distribution $\mu_{0, c^*}$ is constructed by $R(c^*) + \Gaussian{0}{I}$ where $R$ is given as in training. 

\begin{figure}[ht]
\centering
    \begin{minipage}{\textwidth}
        \begin{algorithm}[H]
            \caption{Generation using the matrix field $u_\theta$}
            \begin{algorithmic}
               \renewcommand{\algorithmicrequire}{\textbf{Input:}}
               \REQUIRE Trained $u_\theta$,  source distribution $p_{0,0}$, target condition $c_\ast$, 
               \renewcommand{\algorithmicensure}{\textbf{Return:}}
               \ENSURE A sample $x_1$ from $p\given{\cdot}{c_\ast}$
               \STATE Sample $z$ from source distribution $p_{0,0}$
               \STATE Solve the regression problem $R\colon c \mapsto \operatorname{Mean}[{D}_c]$ on $C$
               \STATE Set $x_{0,c} = z + R(c)$
               \STATE Return $\ODEsolve\qty(x_{0,c}, u_\theta(\blank, c, \blank) \sbmqty{1\\0_k})$
            \end{algorithmic}
            \label{alg:gen}
        \end{algorithm}
    \end{minipage} 
    \begin{minipage}{\textwidth}
        \begin{algorithm}[H]
            \caption{Transfer using the matrix field $u_\theta$}
            \begin{algorithmic}
               \renewcommand{\algorithmicrequire}{\textbf{Input:}}
               \REQUIRE Trained Network $u_\theta$,  source sample $x_0\sim p_{1,c_1}$ with condition label $c_1$, target condition $c_2$
               \renewcommand{\algorithmicensure}{\textbf{Return:}}
               \ENSURE A sample $x_2$ from $p\given{\cdot}{c_2}$ 
               \STATE Return $\ODEsolve(x_0, u_\theta(1, \gamma^{c_1 \to c_2}(\blank), \blank) \sbmqty{0\\c_2-c_1})$
               \begin{flushright}
                  \texttt{\color{teal}\# $\gamma^{c_1\to c_2}$ is defined in \cref{ex:style_trans}}
               \end{flushright}
            \end{algorithmic}
            \label{alg:transfer}
        \end{algorithm}
    \end{minipage}
\end{figure}

\section{Related Works} \label{sec:related} 

Since the debut of \citet{lipman2023flow}, several studies have explored ways formalize the application of flow-based models to conditional generation tasks. 
Some works \citep{dao2023flow, zheng2023guided} take the approach of parametrizing the vector field $v$ with the conditional value $c$ together with the so-called guidance scale $\omega\in\R$ in the form of $v(t,c,x) =  \omega v_t\given{x}{\varnothing} + (1- \omega)   v_t\given{x}{c},
$
which is inspired by the classifier-free guidance scheme of \citet{ho2022classifier}.
\citet{zheng2023guided} in particular has shown that if
$v_t\given{x}{c}$ in this expression well-approximates the conditional score $\nabla \log p\given{x}{c}$, then with the appropriate choice of $\omega$, $v_t(x, c)$ does correspond to the sequence of probability distributions beginning from the standard Gaussian distribution and ending at the target distribution.  
The success of this scheme hinges on the quality of the approximation of the conditional score, and it is reported \citep{lipman2023flow} that in image applications, a guidance scale with a range from $1.2$ to $1.3$ yields competitive performance in terms of FID.
Meanwhile, \citep{hu2023latent} takes the approach of creating a guidance vector by the average of $v_t(x_{c_{\mathrm{targets}}}) - v_t(x_{c_{\mathrm{others}}})$. 
Like the naive application of OT-CFM to conditional generation that simply concatenate the conditional value to the input of the network modeling the vector field, however, these approaches does not allow the user to control the continuity of generated $\mu_c$ with respect to $c$, except through the blackbox architecture of the network modeling $v$. 

Unlike these approaches, \ourmethod\ constructs the flow of generation for an arbitrary condition $c \in \Omega$ through the matrix field $u\colon \tcset\times D \to \R^{d \times (1+k) }$ which solves GCE, or the system of continuity equations defined over $\tcset$, and one can introduce an inductive bias to the continuity of $\mu_c$ with respect to $c$ through the design of the distribution $Q$ of $\psi$ used in the objective function.  
The Dirichlet energy that we use in the demonstration of \ourmethod  is akin to the control of Lipschitz constant for $\psi$ and hence $\mu$, except that it also comes with the boundary condition to assure the generation of the conditional distributions used at the time of the training. 
Also, when $u$ is trained with the random conditional paths with appropriate boundary conditions, our EFM theory in \cref{sec:EFMtheory} guarantees that the flow $\phi^{\gamma^c}$ in \cref{ex:cond_gen} transforms the source distribution to the target conditional distribution whenever $c$ is a condition used in the training. 
Also related to \citet{chen2023riemannian}, which uses a riemannian geodesic to model the distribution of $\psi: I \to M$, where $M$ is the manifold.  


We shall also mention the family of methods based on the Shr\"{o}dinger bridge, which also aims to interpolate between an arbitrary pair of distribution
\citep{tong2023simulation}. 
This direction can be regarded as the problem of solving the continuity equation while minimizing the regularized energy of user's choice \citep{koshizuka2022neural} in the generation process. 
\citep{kim2023wasserstein} also uses Wasserstein Barycenter for distributional interpolation. Stochastic interpolants \citep{albergo2023stochastic} learns a model that is similar to generalized geodesics \cref{sec:ggc-EFM-sample} and it aims to optimize the path in the space of conditions with respect to kinetic energy. 
This approach, however, neither simultaneously models the generation process along $I$ nor formulates the energy with respect to $\mu\colon\Omega \to \Prob{D}$ itself.    

\section{Experiments}\label{sec:exp}
We conducted the following experiments to investigate \ourmethod\ in applications. 
\subsection{Synthetic 2D point clouds}
\label{subsec:synthetic}

We first demonstrate the performance of our method on a conditional distribution consisting of synthetic point clouds in a two-dimensional domain $D\subset\R^2$. Here, we consider the case where the space $\Omega$ of the condition is square, i.e., $\Omega=[0,1]^2$, and train the model when only samples from the conditional distributions $p\given{\cdot}{c}$ at the four corner points $c$ of the square $\Omega$ can be observed (Fig \cref{fig:gt_2dsyn}.)

\begin{figure}[h]
    \centering
        \includegraphics[scale=.5]{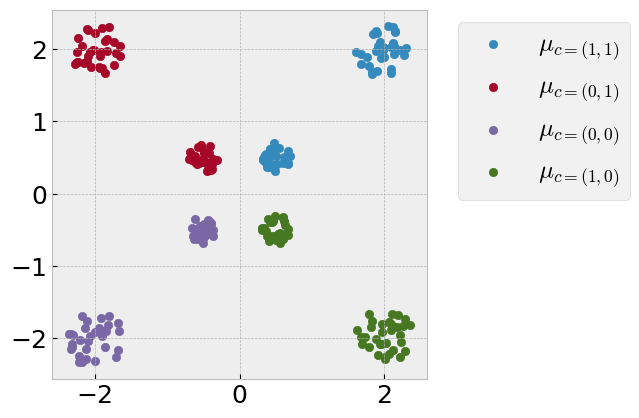}
        \label{fig:gt_2dsyn}
    \caption{Training distributions for the 2D synthetic experiments (4 conditions, two clusters each) }
\end{figure}

We compared our method against COT-FM \citep{chemseddine2024conditional,kerrigan2024dynamic}, as well as OT-CFM and the EFM with the plan$\pi$ which is constructed in the way of generalized geodesic, see \cref{sec:ggc-EFM-sample}.
See Fig \ref{fig:2dsyn} for the generation and transfer visualizations, and see Fig\ref{fig:W1error} for the error between GT and predicted distributions. Note that our method, MMOT-EFM, performs competitively with all its rivals in interpolation and generation tasks. 
Also note that the style transfer with MMOT-EFM preserves the structure of the inner and outer clusters, just as mentioned in the introduction.

\begin{figure}[htb]
    \centering
    \includegraphics[width=0.5\textwidth]{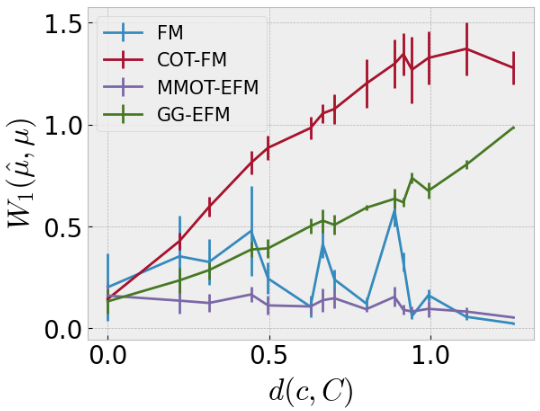}
      \caption{Wasserstein distance between GT vs predicted distributions. COT-FM was evaluated with $\beta=5$.}
    \label{fig:W1error}
\end{figure}

\begin{figure}[htbp]
  \begin{center}
    \begin{subfigure}{\textwidth}
      \centering
      \includegraphics[width=0.75\textwidth]{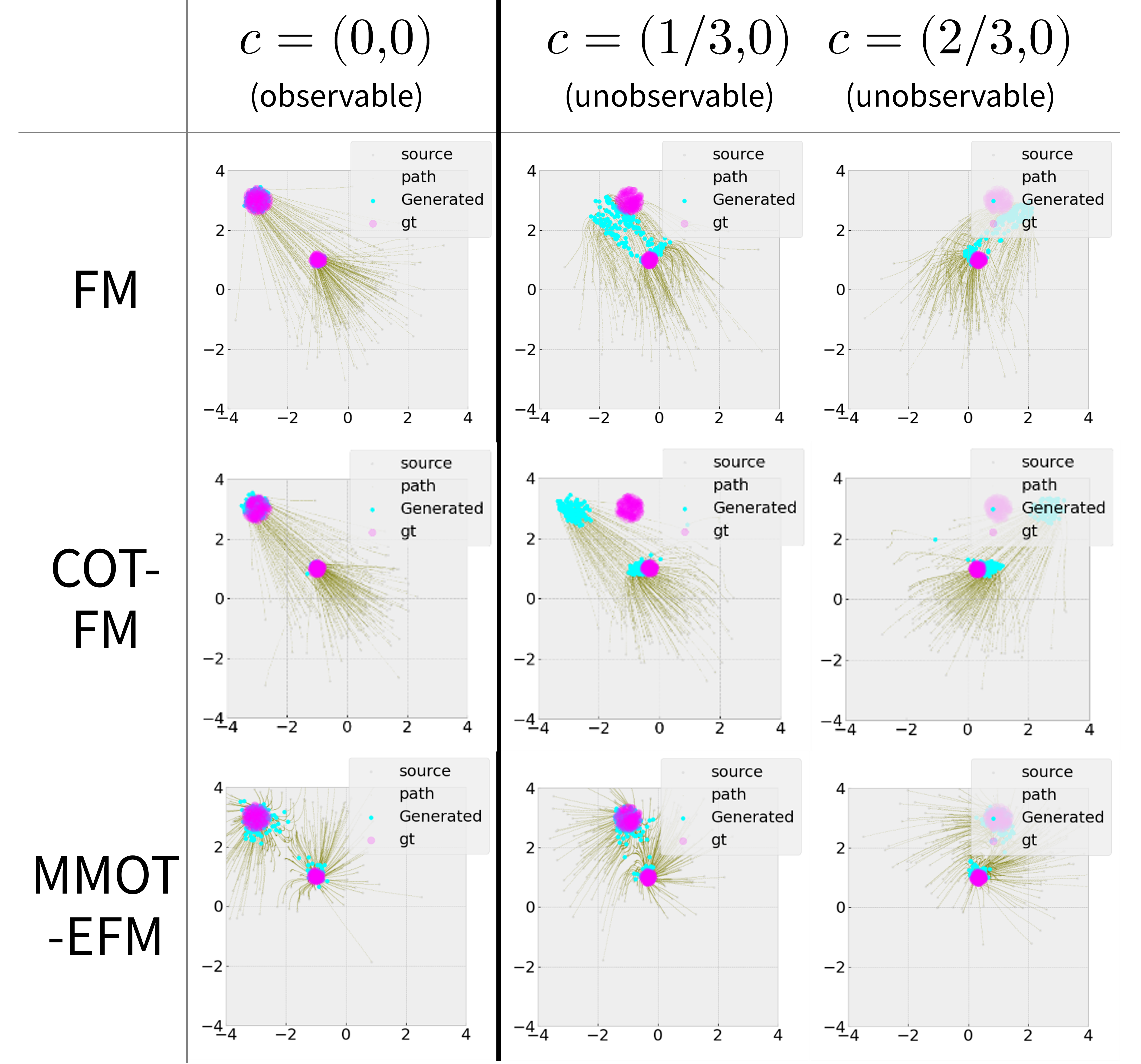}
      \caption{Generated points}
    \end{subfigure}
    \\
    \begin{subfigure}{\textwidth}
      \centering
      \includegraphics[width=0.5\textwidth]{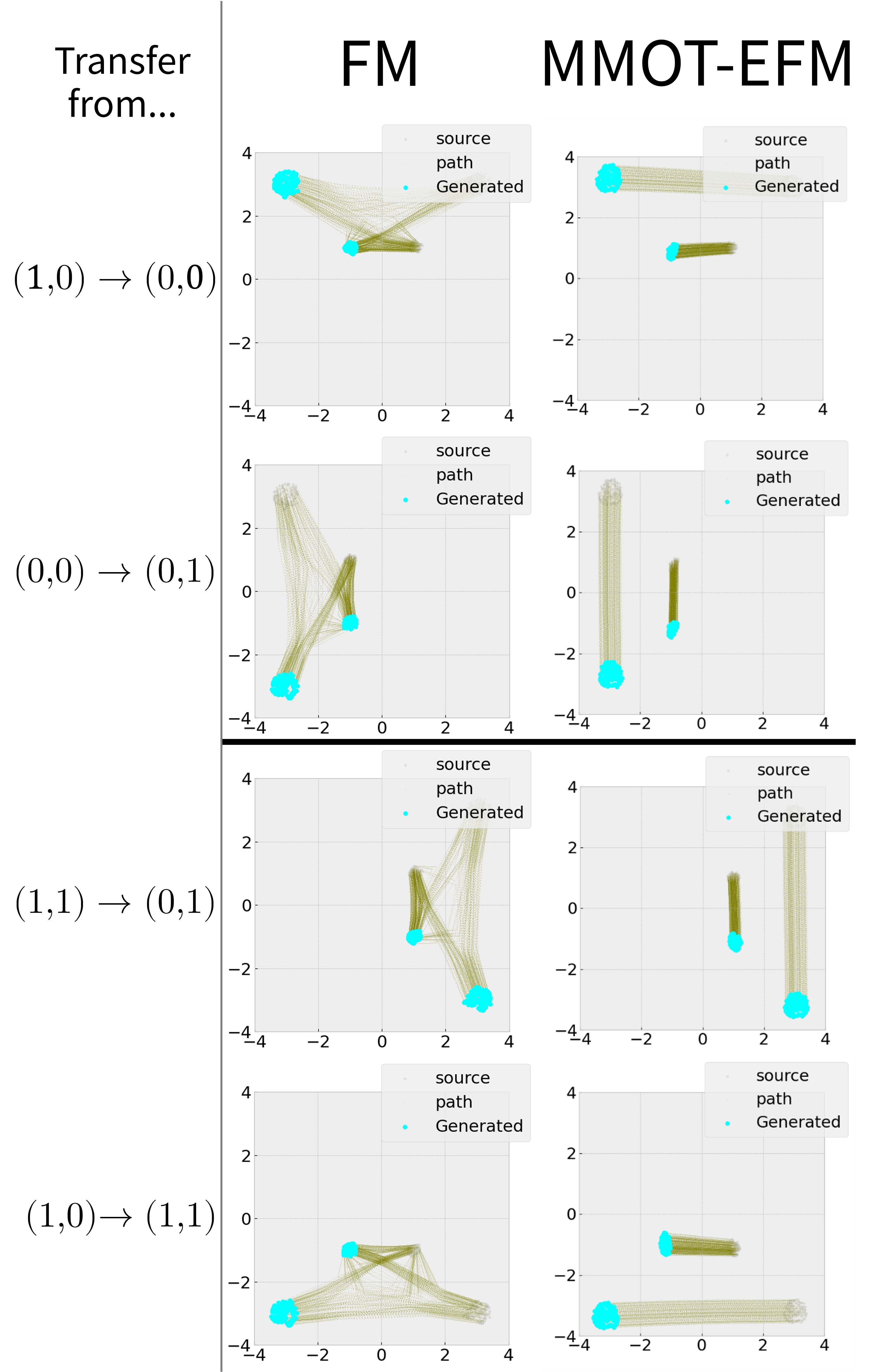}
      \caption{Transfer}
    \end{subfigure}
  \end{center}
  \caption{Conditional Generation results by various models.
  Predicted and GT distributions are colored blue and red, respectively, and the generation/transfer trajectories are drawn in yellow.
  }
  \label{fig:2dsyn}
\end{figure}

\subsection{Conditional molecular generation}
In molecular design applications, multiple chemical properties are often required to be considered simultaneously, and most traditional molecular design methods combine all property requirements and their constraints in a single objective function. We applied MMOT-EFM to the conditional generation for two simultaneous properties of (1) the number of rotatable bonds and (2) the number of Hydrogen Bond Acceptors(HBA). We describe more details on the experiment setup in \cref{sec:exp_detail_cond_mol_gen}. We first trained a VAE model to encode molecular structure into a $32$-dimensional latent space and then trained EFM to perform conditional generation over latent space. 
As shown in \ref{tab:Cond_mol_gen}, our method outperforms all baseline methods on the averaged MAE for both interpolation and extrapolation cases.
\begin{table}[htbp]
\centering
\caption{MMOT-EFM vs. baselines in conditional molecular generations.
}
\label{tab:Cond_mol_gen}
\begin{tabular}{@{}lcc@{}}
\toprule
 & Interpolation MAE  &  Extrapolation MAE \\ \midrule
FM & $1.081 \pm 0.167$ & $1.647 \pm 0.281$\\
COT-FM & 1.023 $\pm$ 0.179 &  1.453 $\pm$ 0.284 \\ \midrule
\textbf{MMOT-EFM(ours)} & \textbf{0.974 $\pm$ 0.137}  &  \textbf{1.344 $\pm$ 0.197}  \\
\bottomrule
\end{tabular}
\end{table}
\section{Conclusion}
In this paper, we developed the theory of EFM, a direct extension of FM that learns the transformation of distributions along the conditional direction as well as along the direction of generation through the modeling of a matrix field instead of a vector field.
EFM models how the distribution changes with respect to conditions in a more explicit form.
We provide the mathematical theory of EFM together with MMOT-EFM, an extension of OT-CFM, with the aim of minimizing the average generation sensitivity with respect to continuous conditions and demonstrating its competitiveness. 
However, we shall note that our current algorithm is limited by the computational cost of MMOT, which grows exponentially with the number of conditional distributions to be used at each step of the algorithm ($\abs{C_0}$).
An advance in the efficient MMOT method or its alternative may significantly improve the scope of applications of EFM. 
The EFM theory is complementary to many powerful existing ideas, particularly through the design of $\psi$ and $Q$, into which one may incorporate the structure of the space of conditions. 
Application to more complex datasets and incorporation of prior knowledge regarding the structure of $\Omega$ is an important future work.  Finally, we note that our theory pertains to the generation of conditional distributions of unseen conditions and interpolation of distributions. 
We shall be aware that, without strong prior knowledge, the identification of unseen distribution 
is an ill-defined problem, and its solution also depends on the architectures and heuristics used therein, as well as the dataset used in the  training.

\bibliography{references.bib}

\begin{thebibliography}{39}
\providecommand{\natexlab}[1]{#1}
\providecommand{\url}[1]{\texttt{#1}}
\expandafter\ifx\csname urlstyle\endcsname\relax
  \providecommand{\doi}[1]{doi: #1}\else
  \providecommand{\doi}{doi: \begingroup \urlstyle{rm}\Url}\fi

\bibitem[Akhmetshin et~al.(2021)Akhmetshin, Lin, Mazitov, Ziaikin, Madzhidov, and Varnek]{Akhmetshin2021}
Tagir Akhmetshin, Arkadii~I. Lin, Daniyar Mazitov, Evgenii Ziaikin, Timur Madzhidov, and Alexandre Varnek.
\newblock {ZINC 250K data sets}.
\newblock 12 2021.
\newblock \doi{10.6084/m9.figshare.17122427.v1}.
\newblock URL \url{https://figshare.com/articles/dataset/ZINC_250K_data_sets/17122427}.

\bibitem[Albergo et~al.(2023)Albergo, Goldstein, Boffi, Ranganath, and Vanden-Eijnden]{albergo2023stochastic}
Michael~S Albergo, Mark Goldstein, Nicholas~M Boffi, Rajesh Ranganath, and Eric Vanden-Eijnden.
\newblock Stochastic interpolants with data-dependent couplings.
\newblock \emph{arXiv preprint arXiv:2310.03725}, 2023.

\bibitem[Ambrosio et~al.(2008)Ambrosio, Gigli, and Savar{\'e}]{AGS}
Luigi Ambrosio, Nicola Gigli, and Giuseppe Savar{\'e}.
\newblock \emph{Gradient flows in metric spaces and in the space of probability measures}.
\newblock Lectures in Mathematics ETH Z\"urich. Birkh\"auser Verlag, Basel, 2 edition, 2008.

\bibitem[Bose et~al.(2023)Bose, Akhound-Sadegh, Fatras, Huguet, Rector-Brooks, Liu, Nica, Korablyov, Bronstein, and Tong]{bose2023se}
Avishek~Joey Bose, Tara Akhound-Sadegh, Kilian Fatras, Guillaume Huguet, Jarrid Rector-Brooks, Cheng-Hao Liu, Andrei~Cristian Nica, Maksym Korablyov, Michael Bronstein, and Alexander Tong.
\newblock {SE(3)}-stochastic flow matching for protein backbone generation.
\newblock \emph{arXiv preprint arXiv:2310.02391}, 2023.

\bibitem[Brenier(2003)]{Brenier2003}
Yann Brenier.
\newblock \emph{Extended Monge--Kantorovich Theory}, pages 91--121.
\newblock Springer Berlin Heidelberg, Berlin, Heidelberg, 2003.
\newblock ISBN 978-3-540-44857-0.
\newblock \doi{10.1007/978-3-540-44857-0_4}.
\newblock URL \url{https://doi.org/10.1007/978-3-540-44857-0_4}.

\bibitem[Chemseddine et~al.(2024)Chemseddine, Hagemann, Wald, and Steidl]{chemseddine2024conditional}
Jannis Chemseddine, Paul Hagemann, Christian Wald, and Gabriele Steidl.
\newblock Conditional wasserstein distances with applications in bayesian ot flow matching, 2024.

\bibitem[Chen and Lipman(2023)]{chen2023riemannian}
Ricky~TQ Chen and Yaron Lipman.
\newblock Riemannian flow matching on general geometries.
\newblock \emph{arXiv preprint arXiv:2302.03660}, 2023.

\bibitem[Dao et~al.(2023)Dao, Phung, Nguyen, and Tran]{dao2023flow}
Quan Dao, Hao Phung, Binh Nguyen, and Anh Tran.
\newblock Flow matching in latent space.
\newblock \emph{arXiv preprint arXiv:2307.08698}, 2023.

\bibitem[Davtyan et~al.(2023)Davtyan, Sameni, and Favaro]{davtyan2023efficient}
Aram Davtyan, Sepehr Sameni, and Paolo Favaro.
\newblock Efficient video prediction via sparsely conditioned flow matching.
\newblock In \emph{Proceedings of the IEEE/CVF International Conference on Computer Vision}, pages 23263--23274, 2023.

\bibitem[Ding et~al.(2021)Ding, Wang, Xu, Welch, and Wang]{ding2021ccgan}
Xin Ding, Yongwei Wang, Zuheng Xu, William~J Welch, and Z.~Jane Wang.
\newblock Cc{GAN}: Continuous conditional generative adversarial networks for image generation.
\newblock In \emph{International Conference on Learning Representations}, 2021.
\newblock URL \url{https://openreview.net/forum?id=PrzjugOsDeE}.

\bibitem[Durrett(2019)]{Durrett}
Rick Durrett.
\newblock \emph{Probability: Theory and Examples}.
\newblock Thomson, 2019.

\bibitem[Fan and Alvarez-Melis(2023)]{fan2023generating}
Jiaojiao Fan and David Alvarez-Melis.
\newblock Generating synthetic datasets by interpolating along generalized geodesics.
\newblock In \emph{Uncertainty in Artificial Intelligence}, pages 571--581. PMLR, 2023.

\bibitem[Flamary et~al.(2021)Flamary, Courty, Gramfort, Alaya, Boisbunon, Chambon, Chapel, Corenflos, Fatras, Fournier, Gautheron, Gayraud, Janati, Rakotomamonjy, Redko, Rolet, Schutz, Seguy, Sutherland, Tavenard, Tong, and Vayer]{flamary2021pot}
R{\'e}mi Flamary, Nicolas Courty, Alexandre Gramfort, Mokhtar~Z. Alaya, Aur{\'e}lie Boisbunon, Stanislas Chambon, Laetitia Chapel, Adrien Corenflos, Kilian Fatras, Nemo Fournier, L{\'e}o Gautheron, Nathalie~T.H. Gayraud, Hicham Janati, Alain Rakotomamonjy, Ievgen Redko, Antoine Rolet, Antony Schutz, Vivien Seguy, Danica~J. Sutherland, Romain Tavenard, Alexander Tong, and Titouan Vayer.
\newblock Pot: Python optimal transport.
\newblock \emph{Journal of Machine Learning Research}, 22\penalty0 (78):\penalty0 1--8, 2021.
\newblock URL \url{http://jmlr.org/papers/v22/20-451.html}.

\bibitem[Gebhard et~al.(2023)Gebhard, Wildberger, Dax, Angerhausen, Quanz, and Sch{\"o}lkopf]{gebhard2023inferring}
Timothy~D Gebhard, Jonas Wildberger, Maximilian Dax, Daniel Angerhausen, Sascha~P Quanz, and Bernhard Sch{\"o}lkopf.
\newblock Inferring atmospheric properties of exoplanets with flow matching and neural importance sampling.
\newblock \emph{arXiv preprint arXiv:2312.08295}, 2023.

\bibitem[Goodfellow et~al.(2020)Goodfellow, Pouget-Abadie, Mirza, Xu, Warde-Farley, Ozair, Courville, and Bengio]{goodfellow2020generative}
Ian Goodfellow, Jean Pouget-Abadie, Mehdi Mirza, Bing Xu, David Warde-Farley, Sherjil Ozair, Aaron Courville, and Yoshua Bengio.
\newblock Generative adversarial networks.
\newblock \emph{Communications of the ACM}, 63\penalty0 (11):\penalty0 139--144, 2020.

\bibitem[Ho and Salimans(2022)]{ho2022classifier}
Jonathan Ho and Tim Salimans.
\newblock Classifier-free diffusion guidance.
\newblock \emph{arXiv preprint arXiv:2207.12598}, 2022.

\bibitem[Ho et~al.(2020)Ho, Jain, and Abbeel]{ho2020denoising}
Jonathan Ho, Ajay Jain, and Pieter Abbeel.
\newblock Denoising diffusion probabilistic models.
\newblock \emph{Advances in neural information processing systems}, 33:\penalty0 6840--6851, 2020.

\bibitem[Hu et~al.(2023)Hu, Zhang, Tang, Mettes, Zhao, and Snoek]{hu2023latent}
Vincent~Tao Hu, David~W Zhang, Meng Tang, Pascal Mettes, Deli Zhao, and Cees G.~M. Snoek.
\newblock Latent space editing in transformer-based flow matching.
\newblock In \emph{ICML Workshop on New Frontiers in Learning, Control, and Dynamical Systems}, 2023.
\newblock URL \url{https://openreview.net/forum?id=Bi6E5rPtBa}.

\bibitem[Ishitani et~al.(2022)Ishitani, Kataoka, and Rikimaru]{ishitani2022molecular}
Ryuichiro Ishitani, Toshiki Kataoka, and Kentaro Rikimaru.
\newblock Molecular design method using a reversible tree representation of chemical compounds and deep reinforcement learning.
\newblock \emph{Journal of Chemical Information and Modeling}, 62\penalty0 (17):\penalty0 4032--4048, 2022.

\bibitem[Isobe(2023)]{isobe2023convergence}
Noboru Isobe.
\newblock A convergence result of a continuous model of deep learning via {Ł}ojasiewicz--{S}imon inequality.
\newblock \emph{arXiv preprint arXiv:2311.15365}, 2023.

\bibitem[Jin et~al.(2018)Jin, Barzilay, and Jaakkola]{jin2018junction}
Wengong Jin, Regina Barzilay, and Tommi Jaakkola.
\newblock Junction tree variational autoencoder for molecular graph generation.
\newblock In \emph{International conference on machine learning}, pages 2323--2332. PMLR, 2018.

\bibitem[Kerrigan et~al.(2023)Kerrigan, Migliorini, and Smyth]{kerrigan2023functional}
Gavin Kerrigan, Giosue Migliorini, and Padhraic Smyth.
\newblock Functional flow matching.
\newblock \emph{arXiv preprint arXiv:2305.17209}, 2023.

\bibitem[Kerrigan et~al.(2024)Kerrigan, Migliorini, and Smyth]{kerrigan2024dynamic}
Gavin Kerrigan, Giosue Migliorini, and Padhraic Smyth.
\newblock Dynamic conditional optimal transport through simulation-free flows, 2024.

\bibitem[Kim et~al.(2023)Kim, Lee, Choi, Won, and Paik]{kim2023wasserstein}
Young-geun Kim, Kyungbok Lee, Youngwon Choi, Joong-Ho Won, and Myunghee~Cho Paik.
\newblock Wasserstein geodesic generator for conditional distributions.
\newblock \emph{arXiv preprint arXiv:2308.10145}, 2023.

\bibitem[Kingma and Welling(2013)]{kingma2013auto}
Diederik~P Kingma and Max Welling.
\newblock Auto-encoding variational bayes.
\newblock \emph{arXiv preprint arXiv:1312.6114}, 2013.

\bibitem[Klein et~al.(2023)Klein, Kr{\"a}mer, and No{\'e}]{klein2023equivariant}
Leon Klein, Andreas Kr{\"a}mer, and Frank No{\'e}.
\newblock Equivariant flow matching.
\newblock \emph{arXiv preprint arXiv:2306.15030}, 2023.

\bibitem[Koshizuka and Sato(2022)]{koshizuka2022neural}
Takeshi Koshizuka and Issei Sato.
\newblock Neural {L}agrangian {S}chr\"{o}dinger bridge: Diffusion modeling for population dynamics.
\newblock In \emph{The Eleventh International Conference on Learning Representations}, 2022.

\bibitem[Lavenant(2019)]{LAVENANT2019688}
Hugo Lavenant.
\newblock Harmonic mappings valued in the wasserstein space.
\newblock \emph{Journal of Functional Analysis}, 277\penalty0 (3):\penalty0 688--785, 2019.
\newblock ISSN 0022-1236.
\newblock \doi{https://doi.org/10.1016/j.jfa.2019.05.003}.
\newblock URL \url{https://www.sciencedirect.com/science/article/pii/S0022123619301478}.

\bibitem[Lin et~al.(2022)Lin, Ho, Cuturi, and Jordan]{JMLR:v23:19-843}
Tianyi Lin, Nhat Ho, Marco Cuturi, and Michael~I. Jordan.
\newblock On the complexity of approximating multimarginal optimal transport.
\newblock \emph{Journal of Machine Learning Research}, 23\penalty0 (65):\penalty0 1--43, 2022.
\newblock URL \url{http://jmlr.org/papers/v23/19-843.html}.

\bibitem[Lipman et~al.(2023)Lipman, Chen, Ben-Hamu, Nickel, and Le]{lipman2023flow}
Yaron Lipman, Ricky T.~Q. Chen, Heli Ben-Hamu, Maximilian Nickel, and Matthew Le.
\newblock Flow matching for generative modeling.
\newblock In \emph{The Eleventh International Conference on Learning Representations}, 2023.
\newblock URL \url{https://openreview.net/forum?id=PqvMRDCJT9t}.

\bibitem[Moosm{\"u}ller and Cloninger(2020)]{moosmuller2020linear}
Caroline Moosm{\"u}ller and Alexander Cloninger.
\newblock Linear optimal transport embedding: Provable wasserstein classification for certain rigid transformations and perturbations.
\newblock \emph{arXiv preprint arXiv:2008.09165}, 2020.

\bibitem[Piran et~al.(2024)Piran, Klein, Thornton, and Cuturi]{MMOT2024Cuturi}
Zoe Piran, Michal Klein, James Thornton, and Marco Cuturi.
\newblock Contrasting multiple representations with the multi-marginal matching gap.
\newblock In \emph{International conference on machine learning}, 2024.

\bibitem[Reshetnyak(1997)]{Reshetnyak1997}
Yu.~G. Reshetnyak.
\newblock Sobolev-type classes of functions with values in a metric space.
\newblock \emph{Siberian Mathematical Journal}, 38\penalty0 (3):\penalty0 567--583, May 1997.
\newblock ISSN 1573-9260.
\newblock \doi{10.1007/BF02683844}.
\newblock URL \url{https://doi.org/10.1007/BF02683844}.

\bibitem[Sohl-Dickstein et~al.(2015)Sohl-Dickstein, Weiss, Maheswaranathan, and Ganguli]{sohl2015deep}
Jascha Sohl-Dickstein, Eric Weiss, Niru Maheswaranathan, and Surya Ganguli.
\newblock Deep unsupervised learning using nonequilibrium thermodynamics.
\newblock In \emph{International conference on machine learning}, pages 2256--2265. PMLR, 2015.

\bibitem[Song et~al.(2020)Song, Meng, and Ermon]{song2020denoising}
Jiaming Song, Chenlin Meng, and Stefano Ermon.
\newblock Denoising diffusion implicit models.
\newblock \emph{arXiv preprint arXiv:2010.02502}, 2020.

\bibitem[Tong et~al.(2023{\natexlab{a}})Tong, Malkin, Fatras, Atanackovic, Zhang, Huguet, Wolf, and Bengio]{tong2023simulation}
Alexander Tong, Nikolay Malkin, Kilian Fatras, Lazar Atanackovic, Yanlei Zhang, Guillaume Huguet, Guy Wolf, and Yoshua Bengio.
\newblock Simulation-free schr{\"o}dinger bridges via score and flow matching.
\newblock \emph{arXiv preprint 2307.03672}, 2023{\natexlab{a}}.

\bibitem[Tong et~al.(2023{\natexlab{b}})Tong, Malkin, Huguet, Zhang, {Rector-Brooks}, Fatras, Wolf, and Bengio]{tong2023improving}
Alexander Tong, Nikolay Malkin, Guillaume Huguet, Yanlei Zhang, Jarrid {Rector-Brooks}, Kilian Fatras, Guy Wolf, and Yoshua Bengio.
\newblock Improving and generalizing flow-based generative models with minibatch optimal transport.
\newblock \emph{arXiv preprint 2302.00482}, 2023{\natexlab{b}}.

\bibitem[Yim et~al.(2024)Yim, Campbell, Mathieu, Foong, Gastegger, Jim{\'e}nez-Luna, Lewis, Satorras, Veeling, No{\'e}, et~al.]{yim2024improved}
Jason Yim, Andrew Campbell, Emile Mathieu, Andrew~YK Foong, Michael Gastegger, Jos{\'e} Jim{\'e}nez-Luna, Sarah Lewis, Victor~Garcia Satorras, Bastiaan~S Veeling, Frank No{\'e}, et~al.
\newblock Improved motif-scaffolding with {SE(3)} flow matching.
\newblock \emph{arXiv preprint arXiv:2401.04082}, 2024.

\bibitem[Zheng et~al.(2023)Zheng, Le, Shaul, Lipman, Grover, and Chen]{zheng2023guided}
Qinqing Zheng, Matt Le, Neta Shaul, Yaron Lipman, Aditya Grover, and Ricky~TQ Chen.
\newblock Guided flows for generative modeling and decision making.
\newblock \emph{arXiv preprint arXiv:2311.13443}, 2023.

\end{thebibliography}
\bibliographystyle{plainnat}

\newpage
\appendix
\onecolumn


\section{Mathematical description of Extended Flow Matching Theory}
Our aim is to sample from the unknown conditional distribution $\Omega\ni c\mapsto p(\bullet\mid c)\in \Prob{D}$.
We extend the flow matching technique developed in \citep{lipman2023flow} for this aim. 
The technique evolves unconditional probability distributions $\mu_t\in\Prob{D}$, $t\in\qty[0,1]$ from a source distribution $\mu_0$ (such as Gaussian $\Gaussian$) to a target distribution $\mu_1\approx p^{\mathrm{data}}$ by means of a continuity equation.
We then introduce a generalized continuity equation that evolves conditional distributions $\mu_{t,c}$, $t\in\qty[0,1]$, $c\in\Omega$ from source distributions $\mu_0$ to the target distributions $\mu_{t=1,c}\approx p^{\mathrm{data}}(\bullet\mid c)$.

To realize this evolution, this section gives an example of how to construct a (at least approximate) solution of the generalized continuity equation and a design of the source distributions $\mu_{t=0,c}$, $c\in\Omega$.
\subsection{Notations} 
\begin{itemize}
    \item $\la\bullet,\bullet\ra$ is the standard inner product and $\abs{\bullet}\coloneqq\sqrt{\la\bullet,\bullet\ra}$.
    \item $D\ni x=(x^1,\dots,x^q)$; data space
    \item $t\in\qty[0,1]$; generation time
    \item $c\in\Omega\subset\R^p$; conditions in a bounded domain $\Omega$.
    \item $\xi=(\xi^0,\xi^1,\dots,\xi^p)\coloneqq(t,c)\in\Omegatilde\coloneqq\qty[0,1]\times\Omega$.
    \item $x\in D\subset\R^q$; data in a compact subset $D$
    \item For $\varphi\in C^1(\Omegatilde\times D;\R^{p+1})$, write $\Div_\xi\varphi\coloneqq\sum_{i=0}^p\partial_{\xi^i}\varphi^{i}$ and 
    \[
    \nabla_x\varphi\coloneqq
    \begin{pmatrix}
        \partial_{x^1}\varphi^0 & \dots & \partial_{x^1}\varphi^p\\
        \vdots & \ddots & \vdots\\
        \partial_{x^q}\varphi^0 & \dots &\partial_{x^q}\varphi^p
    \end{pmatrix}
    \in\R^{q\times(p+1)}.
    \]
    \item $\Prob{X}$; the space of Borel probability measures on a space $X$, endowed with the narrow topology
    \item $\mathcal{P}_2(X)$; the $L^2$-Wasserstein space
    \item $\delta_x\in\mathcal{P}_2(X)$; the delta measure supported at $x\in X$
    \item $\mu_\bullet\colon\Omegatilde\ni\xi\mapsto\mu_\xi\in\Prob{D}$ conditional probability distribution
    \item $L^2(\Omega;X)$; the Lebesgue space valued in a metric space $X$, see \citep[Definition 3.1]{LAVENANT2019688}
    \item $H^1(\Omega;X)$; the Sobolev space valued in a metric space $X$, see \citep[Definition 3.18]{LAVENANT2019688}. In particular, we set $\Gamma\coloneqq H^1(\Omegatilde;D)$   
    \item $\Diri(\mu)$ is the Dirichlet energy of $\mu\in L^2(\Omega;\Prob{D})$, see \citep[Definition 3.5]{LAVENANT2019688}.
    \item $\Uniform(S)$ is the uniform distribution on a subset $S$ of a Euclidean space with unit mass.
    \item $Q\in\Prob{\Psi}$. We will denote by $\psi$ the sample from a probability distribution $Q$.
    \item $\sigma(X)$ denotes the $\sigma$-algebra of a random variable 
\end{itemize}
Following the notation in \citep{Durrett}, we also use the notation $x \sim p$ to designate that $x$ is sampled from the distribution $p$.

\subsection{Generalized continuity equation}
According to \citep[Definition 3.4]{LAVENANT2019688}, we introduce a distributional solution of a generalized continuity equation formally given as 
\begin{equation}\label{eq:GCE2}
\nabla_\xi\mu(\xi,x)+\Div_x(\mu(\xi,x)v(\xi,x))=0.
\end{equation}
The rigorous sense of \eqref{eq:GCE2} is stated in the following.
\begin{definition}[A distributional solution of the generalized continuity equation]\label{def:GCE}
    A pair $(\mu,v)$ of a Borel mapping $\mu\colon\Omegatilde\to\Prob{D}$ valued in probability measures and a Borel matrix field $v\colon\Omegatilde\times D\to\R^{q\times(p+1)}$ is a \emph{solution of the continuity equation} if it holds that
    \[
    \int_{\Omegatilde}\int_{\R^q}\abs{v(\xi,x)}^2\dd{\mu_\xi}(x)\dd{\xi}<+\infty,
    \]
    and
    \[
    \int_{\Omegatilde}\int_{\R^q}\qty(\Div_\xi\varphi(\xi,x)+\la\nabla_x\varphi(\xi,x),v(\xi,x)\ra)\dd{\mu_\xi(x)}\dd{\xi}=0,
    \]
    for all $\varphi\in C_c^\infty(\Omegatilde\times\R^q;\R^{p+1})$.
\end{definition}
If a solution $(\mu,v)$ of the continuity equation is smooth, a path $\gamma$ on $\Omegatilde$ induces a path on $\Prob{D}$:
\begin{proposition}{Lifting conditional paths to probability paths}{induced_flow}
    Let $(\mu,v)$ be a solution of the continuity equation and $\gamma\colon\qty[0,1]\ni s\mapsto\gamma(s)\in\Omegatilde$ be a continuously differentiable curve in $\Omegatilde$.
    Set $\mu^\gamma\coloneqq\mu_{\gamma(\bullet)}\colon\qty[0,1]\to\Prob{D}$ and $v^\gamma(s,x)\coloneqq v(\gamma(s),x)\dot{\gamma}(s)\in\R^q$ for $(s,x)\in\qty[0,1]\times \R^q$.
    
    Suppose that $\Diri(\mu)<+\infty$ and there exists a probability density $\rho\in C^\infty(\Omegatilde;L^\infty(D))$ of $\mu$ with respect to the Lebesgue measure.  
    
    Then, $(\mu^\gamma,v^\gamma)$ satisfies the continuity equation in the sense of distributions, i.e., 
    \[
    \int_0^1\int_{\R^q}\qty(\partial_s\zeta(s,x)+\la\nabla_x\zeta(s,x),v^\gamma(s,x)\ra)\dd{\mu^\gamma_s(x)}\dd{s}=0,
    \]
    for all $\zeta\in C_c^\infty(\qty[0,1]\times\R^q)$.
\end{proposition}
\begin{proof}
     By \citep[Proposition 3.16]{LAVENANT2019688}, there exists a unique $\varphi(\xi,\bullet)\in H^1(D;\R^{p+1})$ for every $\xi\in\overset{\circ}{\Omegatilde}$ satisfying 
     \begin{equation*}
         \nabla_\xi\rho(\xi,x)+\Div_x(\rho(\xi,x)\nabla_x\varphi(\xi,x))=0,\ x\in\overset{\circ}{D},
     \end{equation*}
     and $v=\nabla_x\varphi$ on $\supp\mu$, where $\overset{\circ}{X}$ is the interior of a subset $X$. 
     Thus, we have
     \begin{align*}
         \partial_s\rho(\gamma(s))+\Div_x(\rho(\gamma(s),x)v^\gamma(s,x))&=\qty(\nabla_\xi\rho(\gamma(s),x)+\Div_x(\rho(\gamma(s),x)v(\gamma(s),x)))\dot{\gamma}(s)\\
         &=\qty(\nabla_\xi\rho(\gamma(s),x)+\Div_x(\rho(\gamma(s),x)\nabla_x\varphi(\gamma(s),x)))\dot{\gamma}(s)\\
         &=0.
     \end{align*}
\end{proof}
\begin{remark}\label{rmk:smoothness}
    The smoothness assumption of \cref{prop:induced_flow} recommends us to use some smooth probability measures as source distributions $\mu_{t=1,c}$, $c\in\Omega$.
\end{remark}
According to \cref{prop:induced_flow} and the well-known fact (see \citep[Proposition 8.1.8]{AGS}), if we want a sample under a certain condition $c\in\Omega$, we can flow samples from a source distribution according to the family $\qty(v^\gamma(s,\bullet))_{s\in\qty[0,1]}$ of vector fields determined from a path $\gamma$ satisfying $\gamma(1)=(1,c)$.

\subsection{Principled mass alignment}
A straightforward generalization of \citep[Theorem 1 and Theorem 3]{kerrigan2023functional} yields the following principle in flow marching theory.
\begin{lemma}[Principled mass alignment lemma]\label{lem:PMA}
Let $\cF$ be a separable (complete) metric space and $P$ be a Borel probability measure on $\cF$.
Let $(\mu^f,v^f)$ be a solution of the continuity equation, in the sense of \cref{def:GCE}, for each $f\in\cF$.
Set the marginal distribution as
\[
\bar{\mu}\coloneqq\int_{\cF}\mu^f\dd{P(f)}.
\]
Assume that 
\[
\int_\cF\int_{\Omegatilde}\int_{\R^q}\abs{v^f(\xi,x)}^2\dd{\mu^f_\xi}(x)\dd{\xi}\dd{P(f)}<+\infty,
\]
and $\mu^f_\xi$ is absolutely continuous with respect to $\bar\mu_\xi$ for $P$-a.e.~$f$ and a.e.~$\xi\in\Omegatilde$.
Then, $(\bar{\mu},\bar{v})$ is also a solution, where
\[
\bar{v}(\xi,x)=\int_{\cF}v^f(\xi,x)\dv{\mu^f_\xi}{\mu_\xi}\qty(x)\dd{P(f)},
\]
for $(\xi,x)\in\Omegatilde\times D$.
Moreover, for another matrix field $u$ satisfying 
\[
     \int_{\Omegatilde}\int_{\R^q}\abs{u(\xi,x)}^2\dd{\bar{\mu}_\xi}(x)\dd{\xi}<+\infty,
\]
we have
\begin{equation}\label{eq:Riemannian_metric}
    \int_{\Omegatilde}\int_{\R^q}\la \bar{v}(\xi,x),u(\xi,x)\ra\dd{\bar{\mu}_\xi}(x)\dd{\xi}=\int_{\cF}\int_{\Omegatilde}\int_{\R^q}\la {v}^f(\xi,x),u(\xi,x)\ra\dd{\mu^f_\xi}(x)\dd{\xi}\dd{P(f)}.
\end{equation}
\end{lemma}
The formula \eqref{eq:Riemannian_metric} leads to \cref{thm:PMA}.

\subsection{Lifting data-valued function to probability-measure-valued function}
In order to construct a solution of the generalized continuity equation, we start to consider a particle-based solution of the continuity equation.

According to \citep[Subsection 3.1]{Brenier2003} and \citep[Section 5]{LAVENANT2019688}, we can easily construct a solution of the continuity equation from a given function $\psi\in H^1(\Omegatilde;D)$.
\begin{lemma}\label{lem:path_makes_sol}
    Let $\psi\in H^1(\Omegatilde;D)$ be a function satisfying
    \[
    \int_{\Omegatilde}\abs{\nabla_\xi\psi(\xi)}^2\dd{\xi}<+\infty.
    \]
    Set $\mu^\psi_\bullet\coloneqq\delta_{\psi(\bullet)}\in H^1(\Omegatilde;\Prob{D})$.
Assume that there exists a matrix field satisfying
\begin{equation}
  v^\psi(\xi,\psi(\xi))=\nabla_\xi\psi(\xi),\label{eq:matrix_field} 
\end{equation}
for $\xi\in\Omegatilde$.
Then, $(\mu^\psi,v^\psi)$ is a solution of the continuity equation.
\end{lemma}

Combining \cref{lem:path_makes_sol,lem:PMA}, we can construct another solution of the continuity equation.

\begin{corollary}[The paths make the solution.]\label{cor:PMAandpath}
    Let $Q\in\Prob{H^1(\Omegatilde;D)}$ be a Borel probability measure, and $(\mu^\psi,v^\psi)$ be a solution defined in \cref{lem:path_makes_sol} $Q$-a.e.~$\psi\in H^1(\Omegatilde;D)$ and 
    \[
    \mu^Q\coloneqq\int_{H^1(\Omegatilde;D)}\mu^\psi\dd{Q(\psi)}
    \] 
    is their marginal distribution.
    Assume that
    \[
        \int_{H^1(\Omegatilde;D)}\int_{\Omegatilde}\int_{\R^q}\abs{v^\psi(\xi,x)}^2\dd{\mu^\psi_\xi}(x)\dd{\xi}\dd{Q(\psi)}<+\infty,
    \]
    and $\mu^\psi\ll\mu^Q$.
    Then, $(\mu^Q,v^Q)$ is also a solution of the continuity equation, where \[v^Q=\int_{H^1(\Omegatilde;D)}v^\psi(\xi,x)\dv{\mu^\psi_\xi}{\mu_\xi}\qty(x)\dd{Q(\psi)}.\]
\end{corollary}

\section{Technical proofs}
The following claim follows immediately from the convexity of the Dirichlet energy as shown in \citet[Proposition 3.13]{LAVENANT2019688} and from Jensen's inequality:
\begin{proposition}{Straightness is controlled by $\boldsymbol{\psi}$}{Dirichlet_Jensen}
Let $\mu_{t,c}=\Expect_{\psi\sim\psidist}\qty[\delta_{\psi(t,c)}]$  $((t,c)\in\tcset)$ with $\eta\in\Prob{D}$.
Then, the Dirichlet energy of $\mu\colon\tcset\to\Prob{D}$ is bounded as
\[
    \Diri_{\tcset}(\mu)\leq\iint\limits_{\tcset}\Expect_{\psi\sim\psidist}\norm{\nabla_{t,c}\psi(t,c)}^2\dd{t\dd c}.
\]
\end{proposition}

\begin{proposition}{}{Dirichlet_bound}
    Let $\mu\in H^1(\Omegatilde;\Prob{D})$ be a {\color{blue} smooth} solution of the continuity equation, and $v\colon\Omegatilde\times\R^q\to\R^{q\times(p+1)}$ is the matrix field associated with $\mu$.
    Assume that $v\in C^1(\Omegatilde\times\R^q;\R^{q\times(p+1)})$ and the derivatives $\partial_c v$, $\partial_xv$ of $v$ is bounded on $\Omegatilde\times\R^q$.
    Then, there exists a constant $C>0$ depend on $p,q$ such that
    \[
        \Diri(\mu(1,\bullet))\leq C\exp(\norm{\partial_xv}_{L^\infty(\Omegatilde\times\R^q;\mathcal{B}(\R^q\times\Omegatilde;\R^q))})\qty(\Diri(\mu(0,\bullet))+\norm{\partial_c v}_{\infty}).
    \]
    Here, $\norm{f}_{\infty}=\sup_{(\xi,x)\in\Omegatilde\times\R^q}\abs{f(\xi,x)}$ for a finite-dimensional valued continuous function $f$ on $\Omegatilde\times\R^q$. 
\end{proposition}
The proof of \cref{prop:Dirichlet_bound} is similar to \citep[Proposition 5.4]{isobe2023convergence}.
\begin{proof}
   By virtue of \citep[Proposition 3.21]{LAVENANT2019688}, we have to estimate
   \[
        \Diri(\mu(1,\bullet))=\lim_{\varepsilon\to0}\frac{C_p}{\varepsilon^{p+2}}\iint_{\Omega^2}W_2^2(\mu(1,c^1),\mu(1,c^2))\dd{c^1\mathrm{d}c^2}.
   \]
   The integrand of the above is decomposed as 
   \begin{align}
       W_2(\mu(1,c^1),\mu(1,c^2))&=W_2\qty(\Phi^{1,c^1}_\#\mu(0,c^1),\Phi^{1,c^2}_\#\mu(0,c^2))\nonumber\\
       &\leq W_2\qty(\Phi^{1,c^1}_\#\mu(0,c^1),\Phi^{1,c^2}_\#\mu(0,c^1))+W_2\qty(\Phi^{1,c^2}_\#\mu(0,c^1),\Phi^{1,c^2}_\#\mu(0,c^2)).\label{eq:diff}
   \end{align}
   Here $\Phi^{t,c}\colon\R^q\to\R^q$ is a flow mapping satisfying
   \begin{equation*}
       \Phi^{t,c}(x)=x+\int_0^tv(s,c,\Phi^{t,c}(x))
       \begin{pmatrix}
       1\\
       0
       \end{pmatrix}
       \dd{s}.
       \label{eq:flow_map}
   \end{equation*}
   The first term of \eqref{eq:diff} is bounded as
   \begin{equation*}\label{eq:first_bound_by_flow}
       W_2\qty(\Phi^{1,c^1}_\#\mu(0,c^1),\Phi^{1,c^2}_\#\mu(0,c^1))^2
       \leq\int_{\R^q}\abs{\Phi^{t,c^1}(x)-\Phi^{t,c^2}(x)}^2\dd{\mu_{0,c^1}(x)}.
    \end{equation*}
    Then, the integrand is also bounded by
    \begin{align*}
       \abs{\Phi^{t,c^1}(x)-\Phi^{t,c^2}(x)}\leq&\int_0^t\norm{v(s,c^1,\Phi^{s,c^1}(x))-v(s,c^2,\Phi^{s,c^2}(x))}_{\mathrm{op}}\dd{s}\\
       \leq&\abs{c^1-c^2}\norm{\partial_c v}_{\infty}\\
       &+\int_0^t\norm{\partial_xv}_{\infty}\abs{\Phi^{t,c^1}(x))-\Phi^{t,c^2}(x))}\dd{s}.
   \end{align*}
   Thus, the Gronwall inequality yields 
   \begin{equation}\label{eq:Gronwall1}
       \abs{\Phi^{t,c^1}(x)-\Phi^{t,c^2}(x)}\leq\abs{c^1-c^2}\norm{\partial_c v}_{L^\infty(\Omegatilde\times\R^q;\mathcal{B}(\Omega\times\Omegatilde;\R^q))}\exp(\norm{\partial_xv}_{L^\infty(\Omegatilde\times\R^q;\mathcal{B}(\R^q\times\Omegatilde;\R^q))}).
   \end{equation}
   By a similar argument, the second term of \eqref{eq:diff} is also bounded as 
   \begin{equation}\label{eq:Gronwall2}
       W_2\qty(\Phi^{1,c^2}_\#\mu(0,c^1),\Phi^{1,c^2}_\#\mu(0,c^2))\leq W_2(\mu(0,c^1),\mu(0,c^2))\exp(\norm{\partial_xv}_{L^\infty(\Omegatilde\times\R^q;\mathcal{B}(\R^q\times\Omegatilde;\R^q))}).
   \end{equation}
   Combining \eqref{eq:Gronwall1} and \eqref{eq:Gronwall2} completes the proof.
\end{proof}

\

\section{Pseudo-codes}

\begin{algorithm}[h]
\caption{Algorithm of OT-CFM}
\begin{algorithmic}[1]
         \renewcommand{\algorithmicrequire}{\textbf{Input:}}
 \REQUIRE Neural Network $v_\theta\colon I\times D\to \R^{d}$, the source distribution $\mu_0$, the dataset $D_\ast\subset D$ from a target distribution $\mu$.
 \renewcommand{\algorithmicensure}{\textbf{Return:}}
 \ENSURE $\theta\in\R^p$
\FOR {each iteration}
 \item[] \texttt{\color{teal}\# Step 1: Sample from datasets}
\STATE Sample a batch $B^0$ from $\mu_0$
\STATE Sample a batch $B^1$ from $D_\ast$
 \item[] \texttt{\color{teal}\# Step 2: Construct $\psi\colon I\to D$}
\STATE Construct an optimal transport plan $\pi$ between $B^0$ and $B^1$
\STATE Jointly sample $(x_0,x_1)\sim\pi$
\STATE Sample $t \sim \Uniform{I}$ 
\STATE Compute 
\[
    \begin{aligned}
    \psi_t&\coloneqq\psi\given{t}{x_0, x_1}\\
    &= (1-t)x_0+tx_1\\
    \dot\psi_t&\coloneqq\dot\psi\given{t}{x_0, x_1}\\
    &= x_1 - x_0
    \end{aligned}
\]
\STATE Update $\theta$ by the gradient of $\|v_\theta(t, \psi_t) - \dot\psi_t\|^2$
  \ENDFOR
\end{algorithmic}

\end{algorithm}

\begin{algorithm}[h]
 \caption{Flow Matching (Training)}
 \begin{algorithmic}[1] 
 \renewcommand{\algorithmicrequire}{\textbf{Input:}}
 \REQUIRE Neural Network $v_\theta\colon I\times D\to \R^{d}$, the source distribution $\mu_0$, the dataset $D_\ast\subset D$ from a target distribution $\mu$.
 \renewcommand{\algorithmicensure}{\textbf{Return:}}
 \ENSURE $\theta\in\R^p$
\FOR {each iteration}
 \item[] \texttt{\color{teal}\# Step 1: Sampling from datasets}
\STATE Sample batches $B^0=\{x_0^i\}_{i=1}^N$ from source $p_0$
\STATE Sample batches $B^1=\{x^j_1\}_{j=1}^N$ from dataset $D_\ast$
 \item[] \texttt{\color{teal}\# Step 2: Constructing a supervisory path $\psi$}
\STATE Construct an optimal transport plan $\pi\in\R^{N\times N}$ between $B^0$ and $B^1$
\STATE Jointly sample $(x_0,x_1)\in B^0\times B^1$ from $\pi$
\STATE Sample $t \in I$ 
\STATE Compute \\
~~(A) $\psi_t\coloneqq\psi\given{t}{x_0, x_1} = (1-t)x_0+tx_1$ \\
~~(B) $\nabla\psi_t\coloneqq\nabla_t \psi\given{t}{x_0, x_1} = x_1 - x_0 $
\item[] \texttt{\color{teal}\# Step 3: Learning vector fields}
\STATE Update $\theta$ by the gradient of $\norm{v_\theta(t, \psi_t) - \nabla\psi_t}^2$
  \ENDFOR
 \end{algorithmic} 
 \label{alg:FM_train}
\end{algorithm}
\begin{algorithm}[h]
 \caption{$\ODEsolve$ for generation}
 \begin{algorithmic}[1] 
 \renewcommand{\algorithmicrequire}{\textbf{Input:}}
 \REQUIRE Initial data $x_0\in D$, vector fields $v\colon I\times D\to \R^{d}$
 \renewcommand{\algorithmicensure}{\textbf{Return:}}
 \ENSURE Terminal value $\phi_1^v(x_0)$  of the solution of ODE $\dot{\phi}_t^v(x_0)=v(t,\phi_t^v(x_0))$
\STATE Compute $\phi_1(x_0)$ via a discretization of the ODE in $t$
 \end{algorithmic} 
 \label{alg:ODEsolve_generation}
\end{algorithm}

\begin{algorithm}[th!]
 \caption{Extended Flow Matching (Training)}
 \begin{algorithmic}[1] 
 \renewcommand{\algorithmicrequire}{\textbf{Input:}}
 \REQUIRE Condition set $C \subset\Omega\subset\R^k$, set of datasets $D_c \subset D\subset\R^d$ for each $c \in C$, network $u_\theta\colon\tcset\times D \to \R^{d \times (1+k)}$, source distributions $p_0\given{\blank}{c}$ ($c\in C$)
 \renewcommand{\algorithmicensure}{\textbf{Return:}}
 \ENSURE $\theta\in\R^p$ 
  \FOR {each iteration}
   \item[] \texttt{\color{teal}\# Step 1: Sampling from datasets}
  \STATE Sample $C_0=\{c_i\}_{i=1}^{N_c} \subset C$ 
  \STATE Sample a batch $B_{0,c}$ from $p_0\given{x}{c}$ for each $c \in C_0$
  \STATE Sample a batch $B_{1,c}$ from $D_c$ for each $c \in C_0$
  \STATE Put $B^0 \coloneqq \{B_{0,c}\}_{c\in C_0}$ and $B^1 \coloneqq \{B_c\}_{c\in C_0}$
 \item[] \texttt{\color{teal}\# Step 2: Constructing supervisory paths $\{\psi_j\}_{j=1}^N$}  
   \STATE Construct a transport plan $\pi$ among $B^0$ and $B^1$ \\
   \begin{flushright}
   \texttt{\color{teal}\# see \cref{sec:EFMalg}}
   \end{flushright}
     \STATE Sample $\{(x_{t,c}^j)_{(t,c)\in \{0,1\}\times C_0}\}_{j=1}^N\subset D^{2N_c}$ from $\pi$
   \STATE For all $j\in[1:N]$, define $\psi_j\colon\tcset\to D$ that regresses $(x_{t,c}^j)_{(t,c)\in \{0,1\}\times C_0}$ on $\{0,1\}\times C_0$
   \begin{flushright}
   \texttt{\color{teal}\# see \cref{eq:psi_tc}}
    \end{flushright}
  \STATE Sample $\{t_k\}_{k=1}^{N_t}\subset I$ 
  \STATE Sample $\{{c}_l^\prime\}_{l=1}^{{N}_c^\prime}\subset\Conv(C_0)$
  \STATE For all $j \in [1:N]$, $k\in [1:N_{t}]$, $l\in [1:{N}_{c}^\prime]$, compute \\
  ~~~(A)  $\psi_{j,k,l} \coloneqq\psi_j(t_k, {c}_l^\prime)$  \\ 
  ~~~(B)  $\nabla\psi_{j,k,l} \coloneqq \nabla_{t, c}  \psi_j(t_k, {c}_l^\prime) $\\  
  \item[] \texttt{\color{teal}\# Step 3: Learning matrix fields}
  \STATE Compute the loss \[L(\theta) = \frac{1}{NN_tN_c^\prime}\sum_{j,k,l} \norm{ u_\theta(t_k, {c}_l^\prime, \psi_{j,k,l}) -  \nabla\psi_{j,k,l} }^2\]
  \STATE Update $\theta$ by the gradient of $L(\theta)$
  \ENDFOR
 \end{algorithmic} 
 \label{alg:train}
\end{algorithm}

\section{Sampling of $\cpsi$ in \eqref{eq:psi_tc}  in \cref{sec:EFMalg} for MMOT-EFM} \label{sec:appendix-sample}

In this section, we follow the notation in \cref{sec:EFMalg} and describe in more detail the construction of $\cpsi( c | \boldsymbol{x}_{C_0}) $ in \eqref{eq:psi_tc}, which is  
\[
  \psi\given{t, c}{x_{0,c},\boldsymbol{x}_{C_0}} = (1-t)x_{0,c} + t \cpsi \given{c}{\boldsymbol{x}_{C_0}}
\]
and the corresponding joint distribution of $\boldsymbol{x}_{C_0} \coloneqq \{x_{i}\}_{c_i \in C_0}$ on $D^{2 |C_0|}$ we used in step 2 of the training algorithm. 
In the final part of this section, we also elaborate how we couple $x_{0,c}$ with $\boldsymbol{x}_{C_0}$.

As we describe in the main manuscript, we introduce our EFM as a direct extension of FM as a method to transform one distribution to another through a learned vector field. 
In particular, we present in this paper an implementation of EFM which extends OT-CFM \cite{tong2023improving}, which aims to train FM as an approximate optimal transport between two distributions (source $\mu_0$ and target $\mu_1$).
To formalize this extension, we need to desribe OT as a minimization of Dirichlet Energy. 

\subsection{OT-CFM as approximate Dicirhlet energy minimization}  \label{sec:appendix-OT-CFM} 

 As is principally described in \cite{LAVENANT2019688}, OT emerges as a coupling of the source $\mu_0$ and the target $\mu_1$ constructed from the constant-speed geodesic (with respect to Wasserstein distance) between $\mu_0$ and $\mu_1$, which can be realized by minimizing the Dirichlet energy 
\begin{align}
\Diri(\mu) =  \inf_{v\colon I\times D \to\R^d} \Set{ \int_{[0,1]\times D} \frac{1}{2} \|v(t, x)\|^2 \mu_t(\dd{x})\dd t | \partial_t \mu_t(x)+ \Div_x(\mu_t(x)v(t,x))=0 } 
\end{align}

over all set of $\mu\colon [0, 1] \to \Prob{D}$ satisfying  $\mu(0) = \mu_0$, $\mu(1) = \mu_1$.  
It is well known that in the standard Euclidean metric space, the minimal energy is achieved by $\mu$ corresponding to $v(t, x)$ that is the derivative of a straight-line of form $\psi^T\given{t}{x} = tT(x) + (1-t)x$ where $T\colon D \to D$, and more particularly as the minimum of
\begin{align}
 \int_{D \times D} \frac{1}{2} \norm{x- y}^2 \pi(\dd x, \dd y) = \int_{D} \frac{1}{2} \|\partial_t \psi^T(t|x) \|^2 (I\times T)_\# \mu_0(dx)
\end{align}
over all $\pi \in \Prob{D \times D}$ with marginal distribution $\mu_0$ and $\mu_1$ or equivalently over all $T$ with $T \#\mu_0 = \mu_1$. 
In OT-CFM, this $\pi$(or $T$) is approximated by the discrete optimal transport solution over a pair of batches $B_0, B_1$ sampled respectively from source and target distributions.
Note that, in this view, $(I\times T)_\# \mu_0$ induces a distribution $Q$ on the path $[0,1] \to D$ generating $\psi^T(t|x)$ with randomness derived from $x$. 

Theorem 3.1 of \cite{yim2024improved} guarantees that the (batch)sample-averaged version of $\mu$ and the (batch)sample-averaged version of $v$ satisfies the continuity equation, thereby yielding the approximation of the dirichlet energy minimizing flow map.

\subsection{MMOT-EFM as approximate Dicirhlet energy minimization} \label{sec:appendix-MMOT-EFM} 

To mimic this construction in multi-marginal setting of EFM, we aim to approximate the solution to the minimization of 
\begin{align}
\Diri(\mu) =  \inf_{v\colon\Omega \times D \to \R^{d \times k}}  \Set{ \int_{\Omega\times D} \frac{1}{2} \|v(c, x)\|^2 \mu_\xi(\dd x) \dd c
| \partial_{c} \mu_\xi(x) + \Div_x(\mu(c,x)v(c,x))=0 } 
\end{align}
over all set of $\mu\colon \Omega \to \Prob{D}$ satisfying  $\mu(c_i) = \mu_i$ for all $c_i \in C_0$.
Note that when $\Omega = [0,1]$, this minimization problem (i.e. Dirichlet Problem) agrees with that of the OT problem on which the method of FM is established. 

Now, in a similar philosophy as FM, we would aim to approximate this Dirichlet energy through multi-marginal optimal transport \cite{MMOT2024Cuturi} over discrete samples. 
Now, under \textit{sufficient} regularity condition (Prop 5.6 \cite{LAVENANT2019688}),
we can similarly argue that there exists some probability $Q$ on the space $\mathcal{F} = H^1(\Omega, D)$ of a map from ``condition'' to ``data'' satisfying  
\begin{align}
\Diri(\mu) = \int_{\Omega \times \mathcal{F}}   \|\partial_c \psi(c)\|^2 Q(\dd \psi) \dd c  
\end{align}
and our goal winds down to finding the energy-minimizing distribution $Q$. 
In this endeavor, we implicitly find $Q$ by specifying a particular space of functions $\mathcal{F}$ and 
generating $\psi\colon \Omega \to D$ from a set of $\{(c_i, x_i)\}_{c_i \in C_0}$ of "condition value" and "observation" for jointly sampled $\{x_i\}_i$ as the regression
\begin{align}
\cpsi(\cdot | \{x_i\}_i) = \arg\min_{\psi \in \mathcal{F}}  \sum_{c_i \in C_0} \| \psi(c_i)  - x_i \|^2    
\end{align}
and minimize the energy with respect to the joint distribution $\pi$ on $D^{|C|}$ from which to sample $\{x_i\}_i$.
That is, we aim to minimize 
\begin{align}
\int \|\nabla_c  \cpsi(c | \{x_i\}_i ) \|^2 \pi(\{dx_i\}_i) \dd c   \label{eq:MMOT_cost} 
\end{align}
with respect to $\pi$. This, indeed, is in the format of MMOT problem, where $c(\{x_i\}_i) := \|\partial_c \psi(c | \{x_i\}_i ) \|^2$.
$\mathcal{F}$ can be chosen for example, as an RKHS or a space of linear function, so that the regression can be solved analytically with respect to $c$.

Just as is done in OT-CFM, we approximate this $\pi$ with the joint distribution over finite tuple of batches $\{B_i\}_i $ with each $B_i$ sampled from $\mu_i$ corresponding to condition $c_i$. 
This approximation is indeed the very $\pi$ that we adopt in MMOT version of our EFM in step 2.

Now, by the virtue of Theorem of principle-mass-alignment \ref{cor:PMAandpath}, we can argue that the (batch)sample-averaged distributions $\mu^\psi$  and the (batch)sample-averaged $v^\psi = \partial_c \psi$ solve the \textit{generalized} continuity equation, thereby yielding the approximation of the Dirichlet energy minimizing map $\mu: \Omega \to \Prob{D}$.     

Note that the above constructions of $\psi \sim Q$ is in complete parallel with that of OT-CFM. See Table\ref{tab:EFM_comparison} for the correspondences.
\begin{table}[htbp]
\centering
\caption{OT-CFM vs MMOT-EFM}
\label{tab:EFM_comparison}
\begin{tabular}{@{}lcc@{}}
\toprule
Framework & OT-CFM  &  MMOT-EFM \\ \midrule
$\mu$ & $[0,1] \to \Prob{D}$ & $\Omega \to \Prob{D}$\\
$\psi$ & $[0, 1] \to  D$ &  $\Omega \to  D$ \\
$v$ & $\partial_t \psi$  &  $\partial_c \cpsi$  \\
$(\mu, v)$ relation & Continuity & Generalized Continuity \\ 
Boundaries & $\{\mu_0, \mu_1\}$  &  $\{\mu_i\}_{c_i \in C_0}$   \\ 
Approximation & OT &  MMOT \\  
\bottomrule
\end{tabular}
\end{table}
We also note that this argument can be extended to $\tilde \Omega = [0, 1] \times \Omega$ in place of $\Omega$. 
However, because of the computational cost of MMOT, we construct our generative model from $\eqref{eq:psi_tc}$, which combines $\cpsi$ and the OT-CFM construction.
In the next section, we elaborate on the construction of the approximation of $\pi$ in \eqref{eq:MMOT_cost} from which to sample $\cpsi$ in \eqref{eq:psi_tc}

\subsection{Approximating MMOT} \label{sec:appendix-MMOT-empirical} 

In general, MMOT is computationally heavy, and even with the advanced methods like the multi-marginal Sinkhorn method developed in \citep{JMLR:v23:19-843}, the computational cost scales as $|B|^{|C|}$, where $|B|$ is the batch size and $|C|$ is the number of conditions to be simultaneously considered. 
To reduce this cost, we took the approach of approximating MMOT through clustering.  More particularly, 
when a batch from $B_i$ is sampled each from $\mu_i$ for condition $c_i$, we applied $K$-means nearest neighborhood clustering (KNN) to $B_i$, yielding sub-batches
$\{U_{ik} \}_{c_i \in C_0, k \in [1:K]}$ with mean values $\{m_{ik}\}_{c_i \in C_0, k \in 1:K}$, where $\cup_{k \in 1:K} U_{ik}  = B_i$.
Let $M_i = \{m_{ik}\}_{k \in [1:K]}$ be the set of cluster-means for batch $i$. Instead of conducting MMOT directly on batch $B_i$, we conduct the MMOT on $\{M_i\}_i$, whose cost will be on the order of $K^{|C|}$.
Applying $\operatorname{argmax}$ operations on the result of MMOT from methods like the Sinkhorn method, we can obtain the deterministic coupling $\pi_m = (\bigtimes_i T_i)_\# \Uniform{M_0}$
where $\Uniform{M_0}$ is the uniform distribution on $M_0$. 
After sampling $m_{0 k^*}  \sim  \Uniform{M_0}$, we couple $U_{i T_i(k^*)}$ with a method of user's choice, where $T_i(k^*)$ is an \textit{abuse of notation} satisfying  
$$m_{i T_i(k^*)} = T_i (m_{0 k^*}). $$
In our implementation of MMOT-EFM, we coupled $\{U_{i T_i(k^*)}\}_i$ with generalized-geodesic coupling as is used in \cite{fan2023generating}, with center distribution being the standard Gaussian with mean being the average of $\{U_{i T_i(k^*)}\}_i$.  
Although we provide a brief description of generalized-geodesic in reference \ref{sec:ggc-EFM-sample}, we would like to refer to \cite{AGS} for a more thorough study.  

Below, we summarize the sampling procedure of of $\{x_i\}_{c_i \in C_0}$ in $\psi(\cdot | \{x_i\}_{c_i \in C_0})$ of MMOT-EFM.
\begin{algorithm}[th!]
 \caption{MMOT sampling with Cluster}
 \begin{algorithmic}[1] 
 \renewcommand{\algorithmicrequire}{\textbf{Input:}}
 \REQUIRE Set of batches $\{B_i\}_i$ with each $B_i$ sampled from $p(\cdot | c_i)$
 \renewcommand{\algorithmicensure}{\textbf{Return:}}
 \ENSURE Joint sample $\{x_i\}_i$from $\{B_i\}_i$
   \item[] \texttt{\color{teal}\# Step 1: Cluster MMOT setup}
  \STATE Cluster each $B_i$ as $\cup_{k \in [1:K]} U_{ik}  = B_i$ with $\operatorname{mean}(U_{ik}) = m_{ik}$
  \STATE Set $M_i = \{m_{ik}\}_{k \in [1:K]}$ 
  \STATE Use MMOT to produce coupling on $\{M_i\}_i$ via $\{T_i\}_i \# \Uniform{M_0}$
    \item[] \texttt{\color{teal}\# Step 2: Sampling}
    \STATE Sample $m_{0k^*}$ from $\Uniform{M_0}$
    \STATE Compute $m_{i T_i(k^*)}  := T_i(m_{0k^*})$
    \STATE Jointly sample from  $\{U_{i T_i(k^*)}\}$ with the method of user's choice, preferrably with deterministic coupling, such as another round of MMOT or generalized-geodesic.
 \end{algorithmic} 
 \label{alg:train}
\end{algorithm}

\subsection{Coupling of $\{ x_{0,c_i} \}_{c_i \in C_0}$ and $\{x_i\}_{c_i \in C_0} $} 
Ideally, it is more closely aligned with the theory of Dirichlet energy to include the source distributions $\{\mu(0, c_i)\}_i$ into the set of distributions to be coupled 
in the MMOT, and enact the argument in \cref{sec:appendix-MMOT-EFM} with $\tilde \Omega = [0, 1] \times \Omega$ in place of $\Omega$. 
As mentioned in the previous section, however, the cost of empirical MMOT scales exponentially with the number of distributions to couples. We, therefore, took an alternative coupling strategy as a computational compromise.  

First, recall from the step 1 of \cref{sec:EFMalg} that $\{ x_{0,c_i} \}_{c_i \in C_0}$ are already coupled with common standard Gaussian sample in the form of $\mu_{0,c} = \operatorname{Mean}[D_c] + \Gaussian{0}{I}$.
To couple $\{ x_{0,c_i} \}_{c_i \in C_0}$  with  $\{x_i\}_{c_i \in C_0} $ which are deterministically coupled through the routine of Section \ref{sec:appendix-MMOT-empirical} as $\{x_i\}_{c_i \in C_0} = \{\mathcal{T}_i (x_0)\}_{c_i \in C_0}$ with $x_0$ sampled from $p( | c_0)$, we may simply couple  $x_{0,c_0}$ with $x_0$ and this will automatically induce the deterministic coupling of $\{ x_{0,c_i} \}_{c_i \in C_0}$ and $\{x_i\}_{c_i \in C_0} $.  
In particular, if $B_{0, c_0}$ is a batch of samples from $p_0(\cdot| c_0)$ and $B_{1, c_0}$ is a batch of samples from $D_{c_0}$ in the step1 of the training, we may couple  $B_{0, c_0}$ with $B_{1, c_0}$ with optimal transport with the methods of user's choice, such as those provided in \cite{flamary2021pot}.

\section{A remark on Generalized Geodesic coupling(\ggc) and the sampling of $\cpsi$ in \eqref{eq:psi_tc}  in \cref{sec:EFMalg} for \ggc-EFM}  \label{sec:ggc-EFM-sample} 

As we have mentioned in Section \ref{subsec:train}, EFM can be defined with any distribution $\psidist\in\Prob{\Psi}$ on the space of functions $\Psi\coloneqq\Set{\psi\colon\tcset\to D|\psi\text{ is differentiable}}$ satisfying the boundary conditions \eqref{eq:BC_psi}.  
We also present still another construction of $\cpsi$ derived from different coupling. 
\subsection{Generalized geodesic coupling}
Generalized geodesic of $\{\mu_i\}$ with base $\nu \in \Prob{D}$, also known in the name of linear optimal transport \cite{moosmuller2020linear} in mathematical literatures, was introduced by \cite{AGS} as 
\begin{align}
\rho_a \coloneqq \left(\sum_{i=1} a_i T_i \right)\# \nu ,  ~~~ a \in \Delta_{m-1}
\end{align}
where $T_i$ is the optimal map from $\nu$ to $\mu_i$ and $\Delta_{m-1}$ is the set of all $\{a_i\}_{i=1}^m $ with $\sum_i a_i = 1$.
This is indeed one of the generalizations to the MacCann's interpolation used in OT between $\mu_0$ and $\mu_1$ through the expression 
$$\rho_t \coloneqq ((1-t)\operatorname{Id} + t T) \# \mu_0,\ t\in [0,1]$$ 
which runs along the geodesic in $\Prob{D}$ with respect to Wasserstein distance. 
Note that $\rho_a$ in Generalized Geodesic  provides not only provides deterministic coupling of $\{\mu_i\}$ through $\rho_{e_i}= {T_i}_\# \nu = \mu_i$, it also interpolates unknown distributions for any $a \in \Delta_{m-1}$. 
We would refer to the deterministic coupling in the form of ${T_i}_\# \nu = \mu_i$ as \ggc-coupling.

\subsection{GGc sampling of $\cpsi$} 
In analogy to the sampling procedure of $\cpsi(\cdot | \{x_i\}_i )$ in MMOT-EFM with MMOT-coupled $\{x_i\}_i$, we may sample $\cpsi(\cdot | \{x_i\}_i )$ with $\{x_i\}_i$ that is jointly sampled with \ggc-coupling.  
We emphasize that $\cpsi$ constructed in such a way does not necessarily minimize an explicit objective as Dirichlet energy and
this might result in EFM with a somewhat erratic style transfer. 
For more empirical investigations, please see the main manuscript.

\section{Experiment details for conditional molecular generation}  \label{sec:exp_detail_cond_mol_gen} 
\subsection{Metrics}
To evaluate our conditional generation, we use the pre-trained VAE model to encode EFM-generated latent vectors into molecules and compute the Mean Absolute Error(MAE) between the generated molecule's property values and the conditioning property values. MAEs are calculated separately for interpolation and extrapolation. All MAEs are first calculated for each property and then averaged for both properties.

\subsection{Dataset and baselines}

We first trained a Site-information-encoded Junction Tree Variational Autoencoder(SJT-VAE) model, which is a variant implementation of Junction Tree Variational Autoencoder(JT-VAE)\citep{jin2018junction}. SJT-VAE was initially designed to eliminate the arbitrariness of JT-VAE and enable applications such as RJT-RL\citep{ishitani2022molecular}. We opted for SJT-VAE over JT-VAE due to its superior reconstruction accuracy and faster training times. However, we see no reason that similar results cannot be reproduced with the original implementation of JT-VAE. 

Our SJT-VAE model is trained on ZINC-250k\citep{Akhmetshin2021}. A subset of 80000 molecules are labeled with the number of HBAs and the number of rotatable bonds. All labels are computed using RDKit. These 80000 molecules are first binned into a 2D matrix based on their labels. 
From this 2D matrix, we selected an area where data are concentrated: the number of HBAs between 3 and 6 and the number of rotatable bonds between 2 and 6. To facilitate the training workflow, training data are up-sampled or capped to 5000 per bin. 
\begin{figure}[htbp]
\centering
\includegraphics[scale=.5]{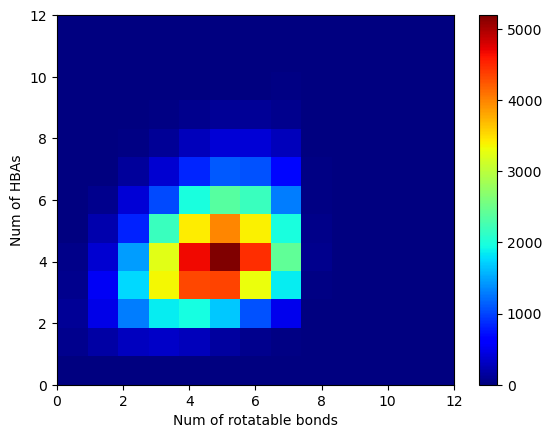}
\label{fig:gt_2dsyn}
\caption{Training set rotatable bonds and HBAs label distribution}
\end{figure}

All flow matching-based models, including MMOTEFM and baselines, are trained with a batch size of 250 and the learning rate of $1\mathtt{e}^{-3}$ for 160,000 iterations. 

\section{Computational Resources}
All models were trained on a single Nvidia V100-16G GPU, and 100 epochs were completed within 4 hours. Training for the MMOT-EFM model is performed on a single Nvidia V100-16G GPU within 2.5 hours.   The results of MMOT-EFM for synthetic experiments were yielded from a model trained over 100000 iterations in 5 hours. 

\section{Additional figures}

  \begin{figure}[ht]
    \centering
    \includegraphics[width=\linewidth]{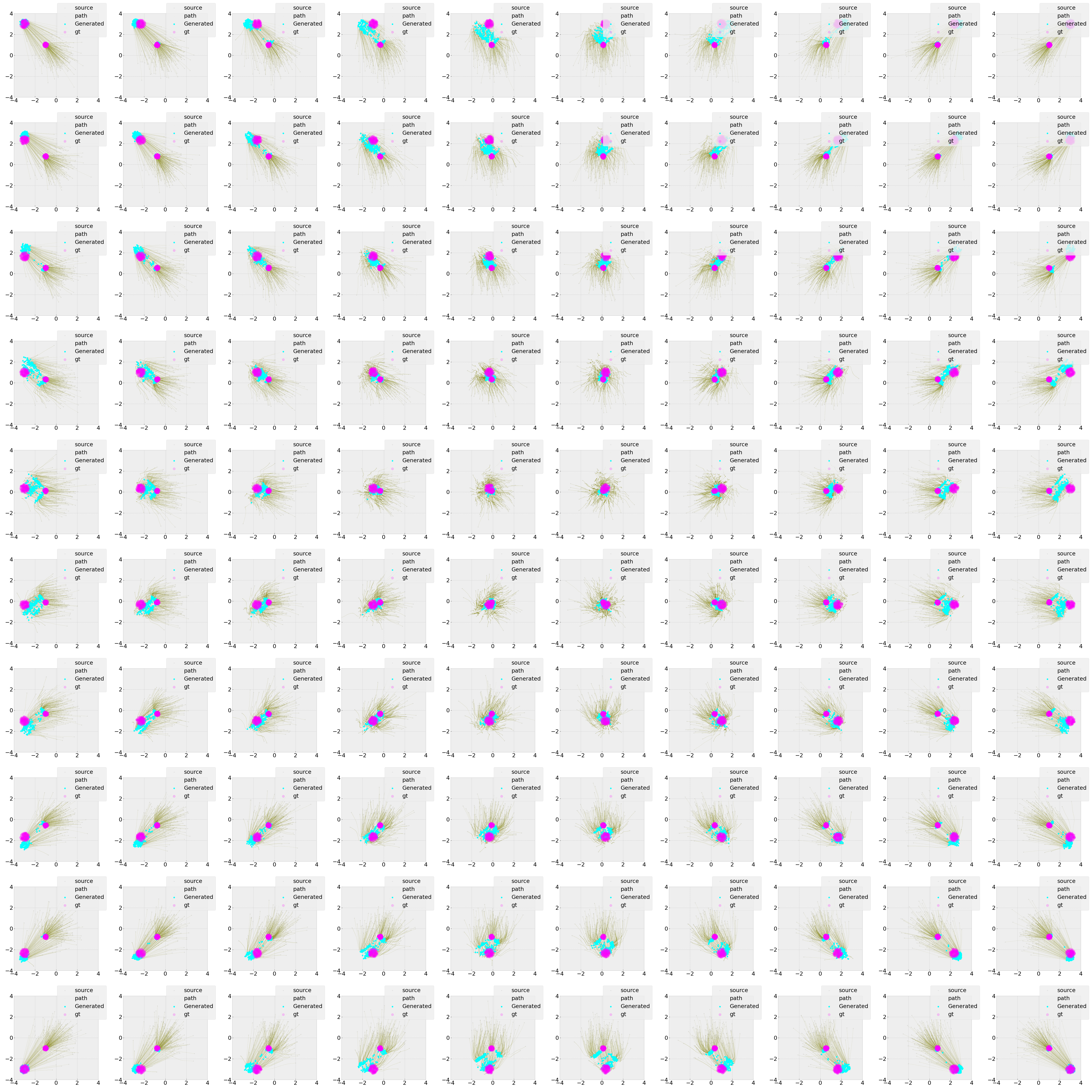}
    \caption{Conditional generation of the synthetic dataset by FM, organized in the grid for two axes of conditions. }
  \end{figure}
  \begin{figure}[ht]
    \centering
    \includegraphics[width=\linewidth]{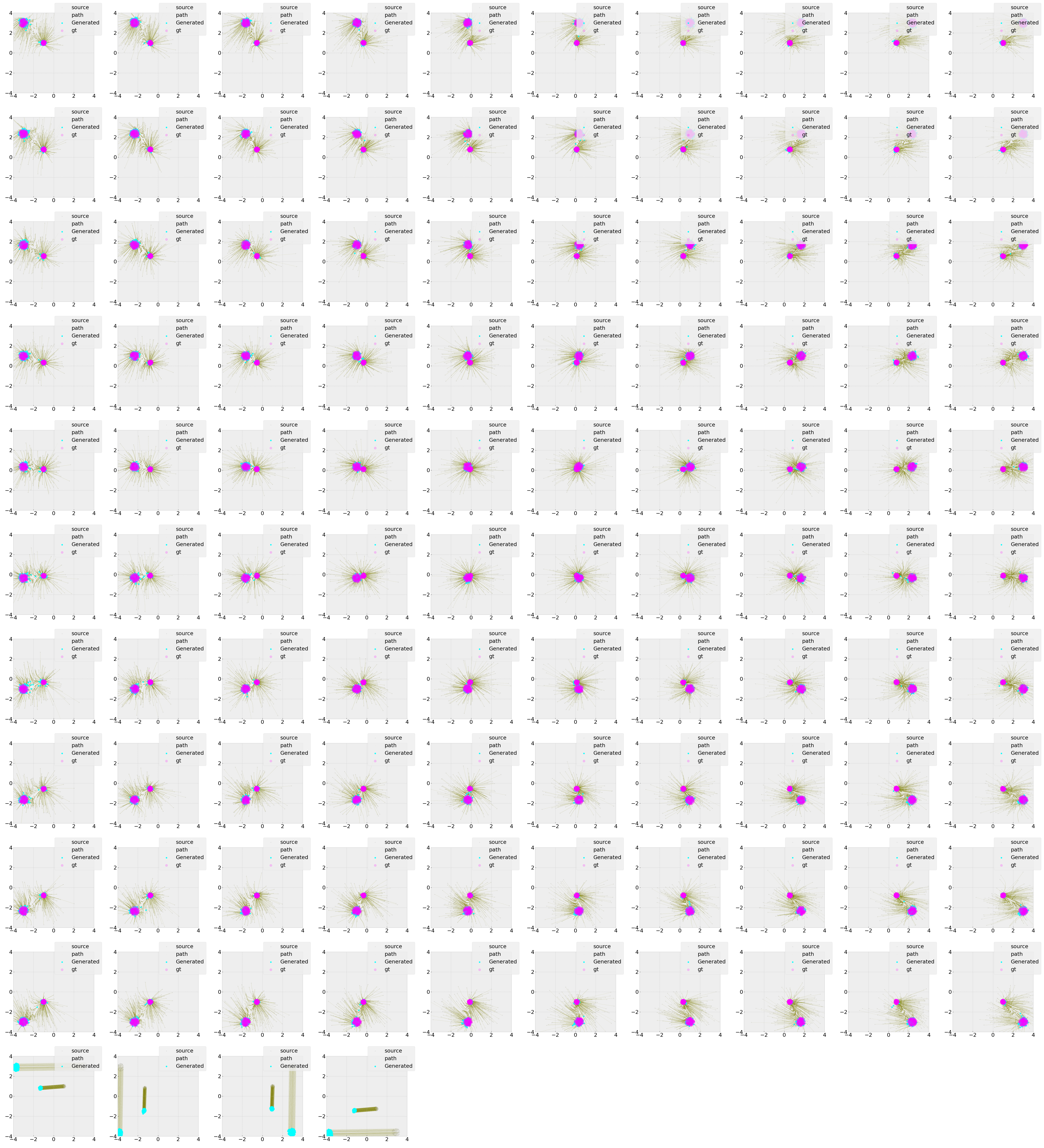}
    \caption{Conditional generation of the synthetic dataset by MMOT-EFM, organized in the grid for two axes of conditions.  The figures in the bottom row are the result of style transfer.}
  \end{figure}


  \begin{figure}[ht]
    \centering
    \includegraphics[width=\linewidth]{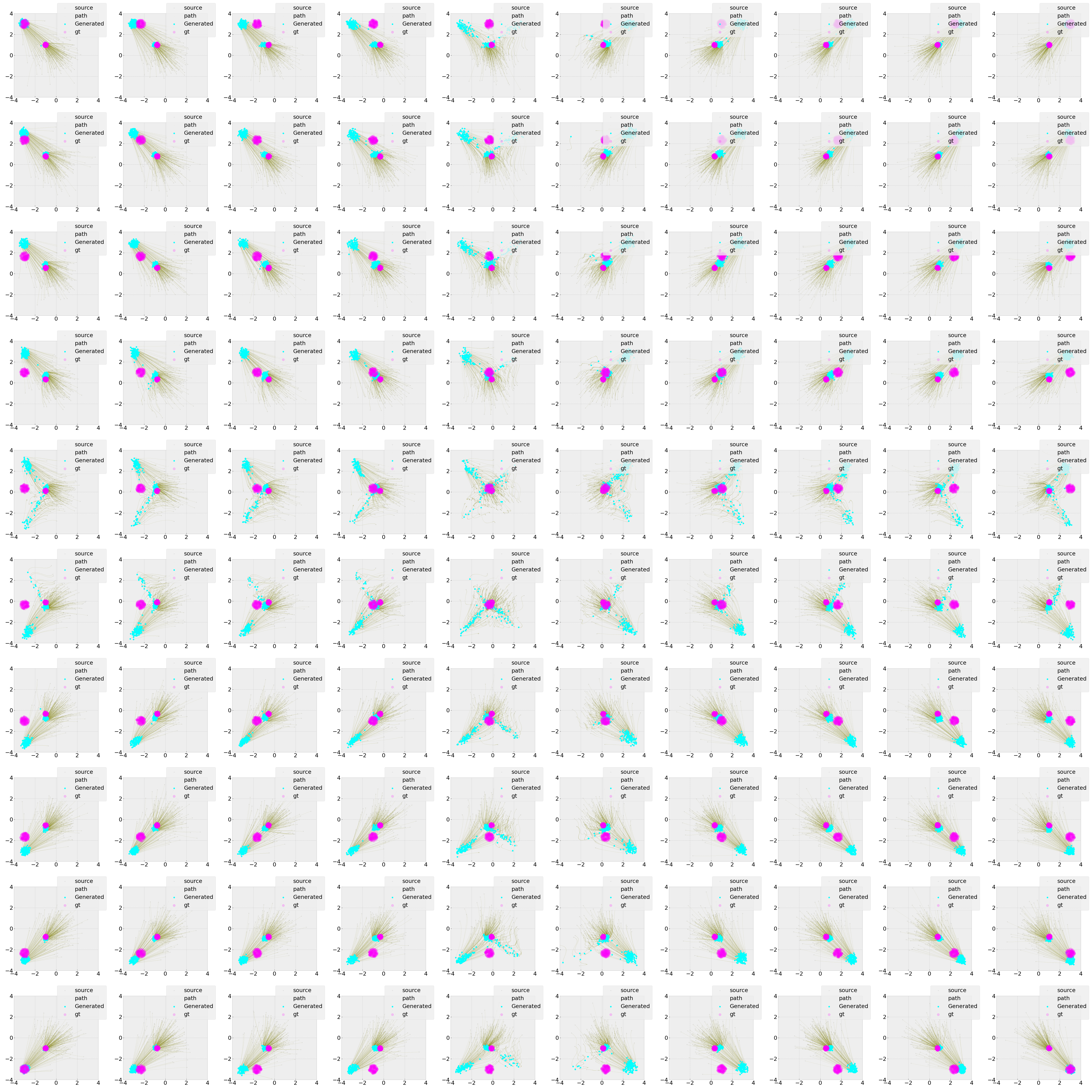}
    \caption{Conditional generation of synthetic dataset by Baysian(COT)-FM with $\beta=10^2$, organized in grid for two axis of conditions.
    }
    \label{fig:betafm_all}
  \end{figure}

\end{document}